\documentclass[11pt]{article}
\usepackage{graphicx} 

\usepackage{fullpage}

\usepackage[framemethod=tikz]{mdframed}

\definecolor{mycolor}{rgb}{0.122, 0.435, 0.698}

\newmdenv[innerlinewidth=0.5pt, roundcorner=4pt,linecolor=mycolor,innerleftmargin=6pt,
innerrightmargin=6pt,innertopmargin=6pt,innerbottommargin=6pt]{mybox}

\usepackage{xurl}
\usepackage{natbib}
\usepackage{bm}
\usepackage[usenames,dvipsnames]{xcolor}
\usepackage{comment}
\usepackage{bbm}
\usepackage{pgfplots}
\pgfplotsset{compat=1.17}
\usepackage{authblk}

\usepackage{tabularx} 
\usepackage{caption} 

\date{}
\author[1]{Steve Hanneke}
\author[2]{Alkis Kalavasis}
\author[3, 4]{Shay Moran}
\author[4]{Grigoris Velegkas}

\affil[1]{Purdue University}
\affil[2]{Yale University}
\affil[3]{Technion}
\affil[4]{Google Research}

\usepackage{amsthm,amsmath,amsfonts,amssymb}
\usepackage{comment}

\usepackage{grigoris}
\usepackage{tikz}
\usepackage{tikz-cd}

\newcommand{\ind}{\mathbf{1}}

\newcommand{\nats}{\mathbb{N}}

\newcommand{\bsigma}{\boldsymbol{\sigma}}

\newcommand{\sD}{\mathscr{D}}

\Crefname{LP}{LP}{LPs}

\usepackage[suppress]{color-edits} 
\addauthor[Shay]{sm}{green}
\addauthor[Steve]{sh}{magenta}
\addauthor[Alkis]{ak}{blue}
\addauthor[Grigoris]{gv}{red}

\title{On the Learning Curves of Revenue Maximization
}

\begin{document}
\pagenumbering{gobble}      

\maketitle

\begin{abstract}
Learning curves are a fundamental primitive in supervised learning, describing how an algorithm’s performance improves with more data and providing a quantitative measure of its generalization ability. Formally, a learning curve plots the decay of an algorithm’s error 
for a fixed underlying distribution as a function of the number of training samples. Prior work on revenue-maximizing learning algorithms, starting with the seminal work of \citet[STOC]{cole2014sample}, adopts a distribution-free perspective (which parallels the PAC learning framework in learning theory). This approach evaluates performance against the hardest possible \emph{sequence} of valuation distributions, one for each sample size, effectively defining the \emph{upper envelope of learning curves} over all possible distributions, thus leading to error bounds that do not capture the shape of the learning curves.

In this work we initiate the study of learning curves for revenue maximization and we provide a near-complete characterization of their rate of decay in the basic setting of a single item and a single buyer. In the absence of any restriction on the valuation distribution, we show that there exists a \emph{Bayes-consistent} algorithm, meaning its learning curve converges to zero for any arbitrary valuation distribution as the number of samples \( n \to \infty \). However, this convergence \emph{must} be arbitrarily slow, {even if the optimal revenue is finite. In contrast, if the optimal revenue is achieved by a finite price} then the optimal rate of decay is roughly $1/\sqrt{n}$. Finally, for distributions supported on {discrete sets of} values,
we show that learning curves decay {(almost)} \emph{exponentially fast}, a rate unattainable under the PAC framework.

From a technical perspective, establishing lower bounds on learning curves is significantly more challenging than in the PAC framework, as it requires fixing a \emph{single hard} distribution and proving a bound that holds for infinitely many values of \( n \). Conversely, deriving upper bounds involves non-trivial algorithmic principles, including techniques such as \emph{regularization} and \emph{structural risk minimization}, which are crucial for achieving optimal learning rates. 
\end{abstract}

\newpage 

\tableofcontents

\newpage

\pagenumbering{arabic}      

\setcounter{page}{1}        

\section{Introduction}\label{sec:intro}

Machine learning has become an integral tool in economics and mechanism design, where it is used to optimize objectives, such as revenue or welfare, using historical data. A prominent application of these techniques is in auction design, particularly in digital advertising, an industry generating hundreds of billions of dollars annually. In this setting, learning algorithms help auctioneers make data-driven decisions to maximize their desired outcomes. This raises a fundamental statistical question: how should we measure the performance of such  learning algorithms?


A common approach that ML practitioners use to measure an algorithm's performance is to plot its \emph{learning curve}, that is, the rate 
of the decay of the ``loss'' of the algorithm under a fixed data generating distribution, as the number of training samples $n \to \infty$ (see red curves in Figure \ref{fig:2}).
Under this framework, a learning algorithm is considered successful (against a distribution \( \D \)) if its learning curve (with respect to \( \D \)) converges to zero, 
and the difficulty of the problem is captured by the \emph{rate} at which the curve approaches zero.
{Learning curves} are a fundamental primitive in statistics and learning theory, dating back to the work of \citet{stone1977consistent} as they quantify how an algorithm’s performance improves with more data. 
Despite their central role in supervised learning, learning curves have not been systematically studied in the context of revenue maximization.

A formal learning theory for auction design, inspired by the PAC learning framework \citep{valiant1984theory,vapnik2013nature}, was put forth in the pioneering work of \citet{cole2014sample} (building on earlier work of \citealp{balcan2008reducing}) and has inspired a long line of beautiful results (e.g.~\cite{huang2015making,morgenstern2015pseudo,roughgarden2016ironing,guo2019settling}). In this framework, the quality of the algorithm is measured by testing its performance against the \emph{worst-case sequence} of distributions, one for \emph{each different sample size} \( n \). 
Thus, PAC learning studies the rate of decay of the \emph{upper envelope} of the learning curves, rather than the rate of decay of individual learning curves.

While PAC learning upper bounds are very desirable since they provide \emph{distribution-free} guarantees, in several settings PAC lower bounds might not be indicative of the behavior of the algorithm of \emph{fixed} distributions.
In this work, we study a complementary learning framework whose goal is to obtain
guarantees on the rate of decay that hold for \emph{every} individual learning curve, but need not necessarily hold for their upper envelope.
To illustrate how our framework complements existing approaches in revenue maximization,
we first recall the standard auction design setting. In the simplest single-buyer single-item problem (which is the main focus of our work), the auctioneer wants to set a price \( p \in \mathbb{R}_+ \) for selling the item to the buyer with the goal of maximizing the expected revenue from this transaction, given by  
\[
\E_{v \sim \D}[p \cdot \ind\{v \geq p\}].
\]
{Here, the revenue equals the posted price \(p\) whenever the buyer’s value meets or exceeds it (i.e., when~\(\ind\{v \ge p\}=1\)).}
We adopt the same statistical model as \citet{cole2014sample}, assuming that~\(\D\) is unknown to the auctioneer and can only be accessed through i.i.d. samples. It is well-known that the revenue-optimal mechanism in this setting, is a \emph{posted-price}: the auctioneer gives a take-it-or-leave-it price to the buyer, who accepts if its valuation is above the offered price.
Thus, a learning algorithm takes as input $n$ samples from the distribution and outputs a price $\hat t_n$ to be posted by the auctioneer. It is natural to define the ``loss'' of the algorithm 
as the gap between the optimal revenue and the revenue of the price it outputs.
This is precisely what the red learning curves in \cref{fig:2} are describing. On the other hand, the PAC learning framework is studying the blue curve in \cref{fig:2}.

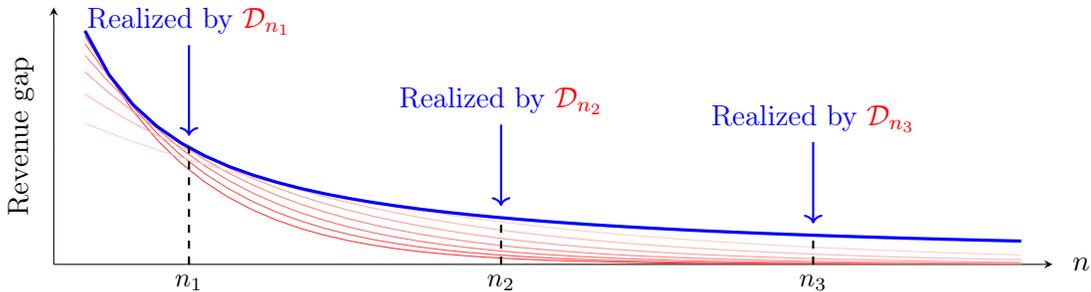
\begin{figure}[!ht]
\centering
\begin{tikzpicture}
\begin{axis}[
    width=0.9\textwidth,
    height=5cm,
    grid=minor,
    xmin=0.7, xmax=10.3, ymax=1.1,
    axis x line=middle,
    axis y line=middle,
    xlabel={$n$},
    ylabel={Revenue gap},
    every axis x label/.style={
        at={(ticklabel* cs:1.01)},
        anchor=west,
    },
    y label style={at={(axis description cs:-.01,.5)},rotate=90,anchor=south},
    yticklabels={,,}, y tick label style={major tick length=0pt},
    xtick={2, 5, 8}, 
    xticklabels={$n_1$, $n_2$, $n_3$}, 
    x tick label style={font=\small}
]

\addplot+[domain=1:10, samples=40, mark=none, color=red!80!white, solid] {0.9*exp(1-0.9*x)};
\addplot+[domain=1:10, samples=40, mark=none, color=red!70!white, solid] {0.8*exp(1-0.8*x)};
\addplot+[domain=1:10, samples=40, mark=none, color=red!60!white, solid] {0.7*exp(1-0.7*x)};
\addplot+[domain=1:10, samples=40, mark=none, color=red!50!white, solid] {0.6*exp(1-0.6*x)};
\addplot+[domain=1:10, samples=40, mark=none, color=red!40!white, solid] {0.5*exp(1-0.5*x)};
\addplot+[domain=1:10, samples=40, mark=none, color=red!30!white, solid] {0.4*exp(1-0.4*x)};
\addplot+[domain=1:10, samples=40, mark=none, color=red!20!white, solid] {0.3*exp(1-0.3*x)};

\addplot+[domain=1:10, samples=40, mark=none, very thick, color=blue, solid] {1/x};

\draw[dashed, thick] (axis cs:2,0) -- (axis cs:2,{1/2});
\draw[dashed, thick] (axis cs:5,0) -- (axis cs:5,{1/5});
\draw[dashed, thick] (axis cs:8,0) -- (axis cs:8,{1/8});

\node[blue] (Dn1) at (axis cs:2, {1/2 + 0.54}) {Realized by \textcolor{red}{$\D_{n_1}$}};
\node[blue] (Dn2) at (axis cs:5, {1/5 + 0.5}) {Realized by \textcolor{red}{$\D_{n_2}$}};
\node[blue] (Dn3) at (axis cs:8, {1/8 + 0.5}) {Realized by \textcolor{red}{$\D_{n_3}$}};

\draw[blue, thick, ->] (Dn1) -- (axis cs:2, {1/2+0.05});
\draw[blue, thick, ->] (Dn2) -- (axis cs:5, {1/5+0.05});
\draw[blue, thick, ->] (Dn3) -- (axis cs:8, {1/8+0.05});



\end{axis}
\end{tikzpicture}
\caption{
Each red curve corresponds to the learning curve of an algorithm against some fixed distribution (see \Cref{def:learning-curve}).
When the worst case distribution is allowed to change with the sample size, we obtain the upper envelope (blue curve) of the learning curves; here, $\D_{n_i}$ is the worst case instance for the given algorithm at the sample size $n_i$. 
}
\label{fig:2}
\end{figure}

At a conceptual level, the learning curves framework complements the PAC framework, offering an alternative perspective for evaluating revenue-maximizing learning algorithms. At a technical level, it introduces new challenges in revenue maximization from samples both for lower and upper bounds.
\begin{enumerate}
    \item (Lower Bounds) Establishing \emph{lower bounds} on learning curves is significantly harder than in the PAC framework, as it requires designing a \emph{single hard} distribution and proving a bound that holds for infinitely many values of \( n \)—a type of guarantee known in statistics as a \emph{strong minimax lower bound}~\citep{antos1996strong}.
    These types of strong minimax lower bounds have been studied in different areas of TCS leading to celebrated results, such as \citet{ben2023new}.
    While the design of hard \emph{sequences} of distributions is quite well understood (with tools such as Fano's inequality or standard constructions using point masses),  crafting a single hard distribution for infinitely many sample sizes requires several new ideas.

    \item (Upper Bounds) 
    At a first sight, upper bounds should be easier in this relaxed framework since any sample complexity achieved by a PAC learner (i.e., all prior works on revenue maximization from samples) directly implies a rate for the learning curves framework. However, these upper bounds might not be \emph{tight}. Indeed, we show several cases where much more optimistic rates can be achieved. For our {upper bounds}, we develop refined variants of the classical Empirical Revenue Maximization (ERM) algorithm, which may have broader applications. Recall that ERM selects a revenue-maximizing price based on the empirical distribution induced by \( n \) samples from \( \D \). 

\end{enumerate}
  
We are now ready to provide a formal treatment of both the PAC model and the learning curves framework.

\subsection{Learning Curves}
We model a learning algorithm as a sequence of (possibly randomized) mappings \( (t_n)_{n \in \N} \) from 
samples drawn from a valuation distribution to prices, i.e., \( t_n: \R_+^n \rightarrow \R_+ \) for sample size \( n \in \N \).
For a given price \( p \), its revenue under \( \D \) ({which is a distribution over valuations}) is given by\footnote{Throughout the paper, we use the convention $F_\D(p) := \Pr_{v\sim \D}[v < p].$}
\[
\rev_\D(p) =  p \cdot (1 - F_{\D}(p)).
\]
{In words, the quantity $1-F_{\D}(p)$ accounts for the mass of the event that the drawn valuation $v$ is (weakly) greater than the price $p$.}
For an algorithm \( (t_n) \), we denote by \( \rev_\D(t_n) \) its revenue, which is a random variable depending on \( S \sim \D^n \) and the internal randomness of the algorithm. The expected revenue is defined as
\[
\E[ \rev_\D(t_n)],
\]
where the expectation is taken over the randomness in \( S \sim \D^n \) and the algorithm. Finally, we define the optimal revenue for \( \D \) as 
\[
\opt_\D = \sup_{p \in \R_+} p\cdot(1- F_\D(p)).
\]
When the distribution is clear from context, we may omit the subscript \( \D \).

A learning curve is a well-established concept in learning theory \citep{stone1977consistent,antos1996strong,bousquet2021theory,schuurmans1997characterizing,antos2001convergence}, and in our setting, it is defined as follows:

\begin{definition}
[Learning Curve]
\label{def:learning-curve}
For any algorithm \( (t_n) \) and a valuation distribution \( \D \), the learning curve of a learning algorithm \( (t_n) \) under \( \D \) is the sequence \((\epsilon_n)\) of expected revenue gaps, where
\[
\epsilon_n(t_n, \D) := \opt_\D - \E [\rev_\D(t_n)].
\]
\end{definition}

An algorithm is a \emph{consistent learner} for a family of distributions \( \mathbb D \) if for any \( \D \in \mathbb D \), it holds that \( \epsilon_n(t_n, \D) \to 0 \) as \( n \to \infty \). For further discussion on consistent learners, see \Cref{proof:BayesConsistency}.  Here, we are interested in the \emph{rate} at which \( \epsilon_n(t_n, \D) \) goes to zero as a function of \( n \), which leads us to contrast the PAC framework with our setting.

\paragraph{PAC Learning.}  
The PAC definition characterizes the \emph{upper envelope} of all learning curves. Specifically, for any sample size \( n \), it takes the pointwise supremum over all valuation distributions in the class. A class of valuation distributions \( \mathbb{D} \) is said to be PAC learnable at rate \( R \) if there exists a learning algorithm \( (t_n) \) such that
\begin{equation}
{\textcolor{red}{(\exists C,c>0)}
\textcolor{blue}{(\forall \D\in \mathbb D)}:
\epsilon_{n}(t_n, \D) \leq C R(cn) \text{ for all }n}.
\label{eq:PAC-UB}
\end{equation}
In words, this ensures that for any sample size \( n \), the worst-case (upper envelope) of the learning curves decays at most as \( C R(cn) \). Crucially, the constants \( c, C \) are \emph{distribution-independent}, meaning the bound holds uniformly over all \( \D \in \mathbb{D} \).

For polynomial rates \( R(n) = n^{-r} \) with \( r > 0 \), the constant \( C > 0 \) alone suffices. However, the additional constant \( c \) accounts for faster rates, such as exponential decay, which we will see are achievable under the less restrictive notion of universal learning (Definition~\ref{def:universal}).

\paragraph{Universal Learning.}
We now introduce a definition that relaxes the uniform requirement in PAC learning and  captures learning curves for individual distributions.

\begin{definition}
[Universal Learning Rate] 
\label{def:universal}
A class of valuation distributions \( \mathbb D \) is \emph{universally learnable} at rate \( R \) if there exists a learning algorithm \( (t_n) \) such that
\begin{equation}
{\textcolor{blue}{(\forall \D\in \mathbb D)}
\textcolor{red}{(\exists C,c>0)}:
\epsilon_{n}(t_n, \D) \leq C R(cn) \text{ for all }n}.
\label{eq:LearningCurve-UB}
\end{equation}
\end{definition}

The key difference from PAC learning (Equation~\eqref{eq:PAC-UB}) is the order of the quantifiers:  
\begin{itemize}
    \item In PAC learning, the bound holds uniformly over all distributions, meaning the constants \( C, c \) are independent of \( \D \).
    \item In universal learning, the constants may depend on the distribution \( \D \), providing a per-distribution guarantee. {The term ``universal''  means that the achieved rate holds for every valuation distribution in the class $\mathbb{D}$, but not \emph{uniformly}
over all distributions.}
\end{itemize}

\paragraph{Strong Minimax Lower Bounds.}  
A key distinction between PAC learning and universal learning lies in how lower bounds are established. In the PAC framework, to derive a lower bound on the (distribution-free) rate, one can construct a sequence of distributions \( (\D_n) \) that vary with the sample size \( n \) and prove the desired lower bound separately for each \( n \).  
In contrast, in the universal framework, a lower bound must hold for \emph{a single valuation distribution} across infinitely many \( n \). That is, one must construct a fixed \( \D \) such that any learning algorithm has a revenue gap of at least \( C R(cn) \) infinitely often.  

Thus, PAC lower bounds have the advantage of adaptively selecting a new worst-case \( \D_n \) at each sample size, whereas universal lower bounds require identifying a single hard instance that remains challenging across all \( n \). This makes universal lower bounds significantly stronger and more difficult to establish.

\subsection{Main Results}
In this section we describe the main results of our work.

\paragraph{Result I} 
Perhaps the most basic question one can ask in this framework is if there exists
a learner whose learning curve converges to zero for \emph{all} underlying distributions. This concept is also referred to as \emph{Bayes-Consistency}.

\begin{defn}[Bayes-Consistency]
Let \(\mathbb D\) be a class of valuation distributions on \(\R_+\). An algorithm \((t_n)_{n\in\N}\) is Bayes-consistent with respect to \(\mathbb D\) if, for every \(\D\in\mathbb D\),
\[
\lim_{n\to\infty}\E[\rev_\D(t_n)] = \opt_\D.
\]
When \(\opt_\D=\infty\), this means that \(\E[\rev_\D(t_n)]\to\infty\).
\end{defn}

Notice that the previous definition allows for distributions whose optimal revenue can be infinite. 
The main question we pose is which are the minimal assumptions on the class $\mathbb D$ that allow for a Bayes-consistent learner. The next result gives a strong positive answer.

\begin{thm}
\label{thm:single-item-bayes-consistency}
    There is a Bayes-consistent learning algorithm for all valuation distributions on $\R_+$.
\end{thm}

This result mirrors classical results of classification and regression in statistical learning theory, where the nearest neighbors algorithm is universally consistent (under assumptions) \citep{stone1977consistent,devroye2013probabilistic}. The algorithm for revenue maximization is quite intuitive: for any sample size $n$, truncate the prices on the interval $(0, \log n)$ and then compute the ERM price on that set of prices. In \Cref{proof:BayesConsistency}, we prove that this strategy eventually converges to the optimal revenue (which is allowed to even be infinite). While this result is not very technically involved, it is an important starting point of our exploration. 

{Note that this result does not imply any learning rate; the guarantee applies only in the limit and does not provide information about the rate of convergence.}
Thus, it is natural to ask what kind of convergence guarantees we can obtain. In order to quantify the rate of decay
of the learning curves, 
we need to consider valuation distributions with finite optimal revenue, as otherwise, the rate is not properly defined. We make this assumption throughout.

\paragraph{Result II} 
{Our next result shows that, with only the assumption that the optimal revenue is finite, no algorithm can ensure anything better than an arbitrarily slow rate of convergence.}

\begin{thm}[Arbitrarily Slow Rates]
    Let $\mathbb{D}$ be the class of valuation distributions $\mathcal{D}$ for which $\opt_\D < \infty$. Then, $\mathbb{D}$ requires \emph{arbitrarily slow} rates. 
\end{thm}

This result provides a rather strong lower bound: for any algorithm \( t_n \) and any rate function~\(R(n)\), there exists a \emph{fixed} hard distribution \( \D \) for which the learning curve \( \epsilon_n(t_n, \D) \) decays no faster than~\(R(n)\). In fact, our hard distribution is supported on $\N.$ 
Our lower bound construction uses a deterministic argument and does not follow the same path as other arbitrarily slow universal rates lower bounds \citep{bousquet2021theory}.
Inspecting the lower construction, an important technical detail becomes apparent: its \emph{optimal revenue is never attained} by any finite price \( p \in \mathbb{R}_+ \) 
(see \Cref{sec:ex-ditr-no-opt-price} for more examples of distributions with this property). It is, thus, natural to ask: does the landscape change if we exclude such distributions?

\paragraph{Result III} 
{We next show that, once the optimal revenue is realized at a finite price, the difficulty underlying our lower bound disappears, allowing for meaningful convergence guarantees.}

\begin{thm}
\label{thm:NearSqRootRates}
Let \( \mathbb{D} \) be the class of distributions \( \D \) s.t.~$\opt_\D <\infty$ and is achievable by a finite price~$p^*_\D$. Then,
\begin{itemize}
    \item for any rate $R(n) \in \omega(n^{-1/2}),$ the class $\mathbb D$ is universally learnable at rate $R(n)$, and
    \item conversely, for any rate \( R(n) \in o(n^{-1/2}) \), the class \( \mathbb{D} \) is not universally learnable at a rate faster than \( R(n) \). Specifically, for any such rate \( R(n) \) and any algorithm \( (t_n) \), there exists a fixed hard distribution \( \D \) such that the revenue gap \( \epsilon_n(t_n, \D) \) is at least \( C R(cn) \) infinitely often.
\end{itemize}

\end{thm}

Recall that $R(n) \in \omega(n^{-1/2})$ means that $\lim_{n \to \infty} R(n) \cdot \sqrt{n} = \infty$. For instance, some valid rates are $R(n) = \log n/\sqrt{n}$ or $R(n) = 1/n^{1/3}$. This means that for any rate arbitrarily close to $1/\sqrt{n}$ (but slower), there is an algorithm that learns $\mathbb D$ at that rate.
We refer the reader to \Cref{def:omega-R-opt} for a formal definition of ``universally learnable at rate $\omega(R(n))$''. Thus, this result gives an (almost) tight bound of $1/\sqrt{n}.$ The lower bound construction is significantly involved, and is, perhaps, the most technically challenging construction in our work. To obtain this result, we create an ensemble of uncountably many hard distributions and, through a probabilistic argument, we show that for any algorithm, one of them will illustrate the slow rates. We elaborate more on it in \cref{sec:tech-overview}.

Interestingly, the above result can be improved for distributions with bounded support distributions, yielding \emph{sharp}
rates for this setting.

\begin{thm}
[Bounded Support]
\label{Thm:Bounded}
Let \( \mathbb{D} \) be the class of all distributions \( \D \) with a bounded support.  
Then, \( \mathbb{D} \) is universally learnable at an optimal \( o(n^{-1/2}) \) rate.
\end{thm}
We refer the reader to Definition~\ref{def:o-R-opt} for a formal definition of ``universally learnable at an optimal \( o(n^{-1/2}) \) rate'' and proceed with a semi-formal explanation. 
This rate is optimal in the sense that:
\begin{itemize}
    \item There exists an algorithm \( (t_n) \) that, for any \( \D \in \mathbb{D} \), estimates the optimal revenue at some rate \( R_\D(n) \) that belongs to the class of \( o(n^{-1/2}) \), where the specific rate depends on \( \D \).  
    \item Conversely, for any rate \( R(n) \in o(n^{-1/2}) \), the class \( \mathbb{D} \) is not universally learnable at a rate faster than \( R(n) \): for any such rate \( R(n) \) and any algorithm \( (t_n) \), there exists a fixed hard distribution \( \D \) such that the revenue gap \( \epsilon_n(t_n, \D) \) is at least \( C R(cn) \) infinitely often.
\end{itemize}

Having established these results, it is natural to ask. Are $\sqrt{1/n}$ rates the best we can hope for under this framework?

\paragraph{Result IV} This brings us to our last main result, which shows that this is not necessarily the case. For that,  we consider the setting where the support of every distribution is a discrete and closed set. Discreteness of the support is natural since in real-world prices are discrete (e.g., multiples of a cent). {A bit more formally,} a subset $\mathcal X$ of $\R_+$ is \emph{closed and discrete} if its elements are isolated and it has no limit points; {e.g., the natural numbers or a uniform $\epsilon$-cover of $\R_+$ are closed and discrete sets; the set $\{1/n : n \in \N\}$ is not closed since it accumulates at 0 while the set  
$\{0\} \cup \{1/n : n \in \N\}$ is not discrete since $0$ is not isolated. Moreover,  
$\mathbb Q \cap \R_+$ is not discrete since every interval contains infinitely many rational points.}

\begin{thm}[Discrete Support]   
\label{Thm:Discrete}   
Let $\mathbb{D}$ be the class of valuation distributions $\mathcal{D}$ whose support is discrete and closed, and such that $\opt_\D < \infty$ is achieved by some $p^*_\D.$ Then,
\begin{enumerate}
    \item $\mathbb{D}$ is universally learnable at an \emph{(almost) exponential rate} $e^{-o(n)}$, but
    \item ERM is \emph{not} a Bayes-consistent learner for $\mathbb{D}.$
\end{enumerate}   
\end{thm}  
 
In this result, the achievable rate is slightly slower than exponential but \emph{arbitrarily close to it}:  
for every function \( R(n) \in o(n) \), we can achieve a rate of \( e^{-R(n)} \) for that class.  
Examples of such functions \( R(n) \) include \( \sqrt{n}, n / \log n, n / \log^* n \).
Notably, exponential rates do not appear in the PAC framework. In fact, standard PAC impossibility results \citep{cole2014sample} rely on sequences of hard distributions supported only on two points. In contrast, our result shows that these distributions admit exponentially fast learning curves under the universal learning framework.

Interestingly, this rate is \emph{not} achievable by ERM,\footnote{Recall that ERM stands for Empirical Revenue Maximization and corresponds to the procedure that, given i.i.d. samples from the valuation distribution, computes the empirical cdf and outputs the price that maximizes the empirical revenue.} but it is achievable by our structured ERM algorithm that is ``biased'' towards smaller prices. In fact, our result shows that ERM is not even converging to the optimal revenue in this setting. To the best of our knowledge, this is the first result in the line of work on revenue maximization from samples that is formally showing that ERM provably fails to learn in a setting where some other algorithm does not. 


It turns out that when the support of the distribution is \emph{finite}, we can achieve \emph{sharp} exponential rates, and ERM is optimal.
\begin{thm}[Finite Support]  
\label{Thm:Finite}  
Let \( \mathbb{D} \) be the class of finitely supported valuation distributions. Then, \( \mathbb{D} \) is universally learnable at an optimal exponential rate \( e^{-n} \).  
\end{thm}  

\paragraph{Summary of the Results}
The results above establish the following landscape, summarized in Figure \ref{fig:venn}, regarding the learning rates for the class of distributions where the maximum revenue is finite:
\begin{enumerate}
    \item If the optimal revenue is \emph{finite but not attained} by any price in the support, then we have shown a strong lower bound: \emph{arbitrarily slow} rates, even when the support is closed and discrete.
    \item When the optimal revenue is \emph{attained} by some price in the support, the structure of the support determines the learning rate:
    \begin{enumerate}
        \item If the support is \emph{finite}, the optimal universal rate is \emph{exponential}.
        \item If the support is \emph{closed and discrete}, the universal rate is almost \emph{exponentially} fast.
        \item If the support is \emph{bounded}, the optimal universal rate converges faster than $1/\sqrt{n}$.
        
        \item If the support is \emph{arbitrary}, the  universal rate converges \emph{almost as fast as $1/\sqrt{n}$}.
    \end{enumerate}  
\end{enumerate}

\begin{figure}
    \centering
\includegraphics[width=0.8\linewidth]{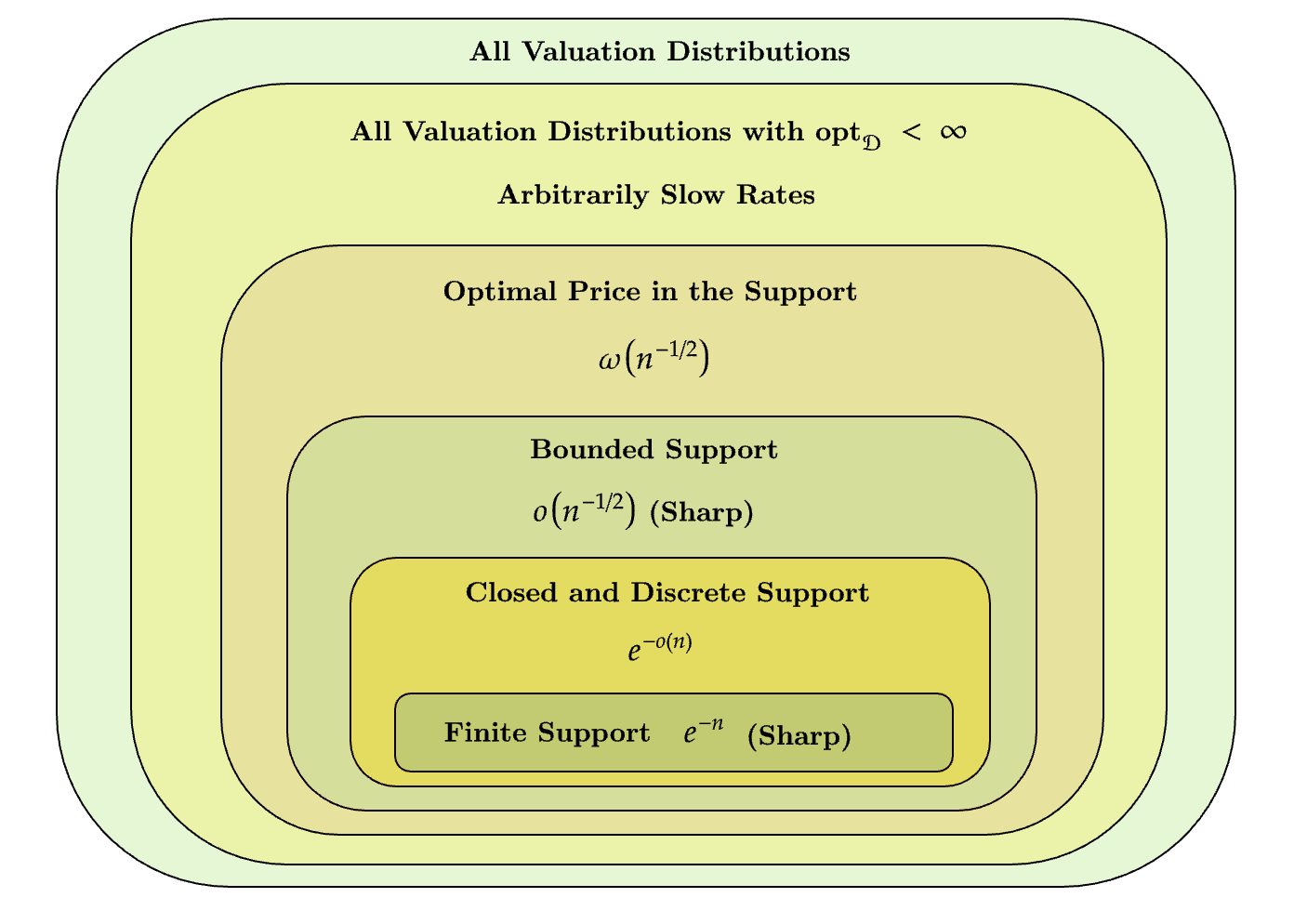}
    \caption{Summary of universal learning rates for revenue maximization in the single-item single buyer setting.}
    \label{fig:venn}
\end{figure}

\paragraph{Comparison with PAC Results}  
{An algorithm that PAC learns a class of valuation distributions \( \mathbb{D} \) provides a strong guarantee, in the sense that it must succeed against \emph{any} sequence of distributions \( (\D_n) \subseteq \mathbb{D} \): at round \(n\), the learner receives \(n\) samples drawn from \( \D_n \), its output is evaluated with respect to \( \D_n \), and the resulting sequence of risks must converge to zero.}
This worst-case nature of the PAC model is reflected in existing learning-theoretic results for revenue maximization. In particular, \citet{cole2014sample} observed that strong structural assumptions on \( \mathbb{D} \) are necessary to obtain positive results in this setting.  

To see why, consider the following family, adapted from \citet{cole2014sample}: \( \mathbb{D} = \{\D_1, \{\D_2^n\}_{n \in \N}\} \), where \( \D_1 \) places all its mass on \( 1 \), and \( \D_2^n \) places \( 1 - \nicefrac{1}{n^2} \) mass on \( 1 \) and \( \nicefrac{1}{n^2} \) mass on \( n^3 \). With high probability, an algorithm drawing \( n \) samples from \( \D_1 \) or \( \D_2^n \) will observe only the value \( 1 \), making it impossible to distinguish whether the optimal selling price should be \( 1 \) or \( n^3 \).  

It is crucial to note that the sequence of hard distributions constructed in PAC lower bounds consists of distributions supported on a finite set (just two points). In contrast, in the universal learning framework, each of these ``hard'' distributions is learnable exponentially fast, as established in \Cref{Thm:Finite}. This indicates that the difficulty in previous PAC lower bounds arises from the adversary’s ability to design a shifting sequence of hard distributions rather than from an inherent challenge in learning from any single fixed distribution.  

In the single-buyer single-item setting, which is the focus of this work, \citet{huang2015making} established near-tight PAC rates for well-structured families of distributions: for monotone hazard rate (MHR) distributions, the PAC rate is \( \widetilde{\Theta}(n^{-2/3}) \); for regular distributions, it is \( \widetilde{\Theta}(n^{-1/3}) \); and for distributions supported on \( [1,H] \), the rate is \( \widetilde{\Theta}(\sqrt{H/n}) \). The universal rates established in this work complement the existing PAC literature by offering a different set of rates under distinct learnability conditions. 
It remains an interesting open question and a natural next step from our work to characterize the optimal rates for such distributions (see also \cref{sec:conclusion}).

As an illustration, the universal rate for distributions with bounded support is asymptotically faster than the PAC learning rate. Moreover, the universal rates for general unbounded distributions with a realized optimal revenue are faster than the PAC rates e.g., for regular distributions.

\begin{rem}
[Additive vs. Multiplicative Error]
We mention that most of prior works have studied revenue maximization with multiplicative error, i.e., aiming at approximating $\opt_\D$ up to a multiplicative $(1-\epsilon)$ factor. We stress that in our learning curves framework, additive and multiplicative approximations are equivalent and do not affect the rate. For more details, we refer to \Cref{sec:additive-mult-error}.
\end{rem}

{
\subsection{Technical Overview}\label{sec:tech-overview}
Before giving formal proofs of our main results,
we believe it is useful to highlight some of technical aspects and conceptual takeaways from our results.

\subsubsection{Lower Bound Constructions} 
Perhaps the most technically challenging results in our work are our lower bounds. As we alluded to, establishing lower bounds in our framework is significantly more challenging than in the PAC framework, since one needs to find a single hard distribution for infinitely many values of the sample size $n.$ We use
three different techniques to establish our three main lower bounds: the arbitrarily slow rates lower bound for general distributions, the $o(n^{-1/2})$ lower for bounded distributions, and the lower bound which shows that ERM does not even converge to the optimal revenue. We additionally show an exponential lower bound on the convergence rates, holding for any non-trivial family of distributions. The proof in this case is more standard.

\paragraph{Arbitrarily slow rates lower bound.} For our lower bound showing that
there are distributions for which arbitrarily slow rates are necessary, we develop a \emph{deterministic} construction that is tailored to the underlying 
learning algorithm. Recall that given a rate $R: \N \rightarrow (0,1)$
and a learning algorithm,
our goal is to construct a distribution
such that for infinitely many sample sizes $\{n_i\}_{i\in \N}$ the revenue gap between the output of the learner and the true revenue is at least $R(n_i).$ 
We achieve that goal by designing a distribution that has discrete support. The first main idea behind our approach
is to try to select the points of the support and their mass
so that two (seemingly conflicting) goals are achieved simultaneously:
{\begin{enumerate}
    \item the mass
of the first $n_i$ points is sufficiently large so that
when the learner draws $n_i$ samples from the distribution
it does not see any points beyond these (with sufficiently high probability), and,
\item the 
optimal revenue of any price among these points is at least $R(n_j)$ far from the optimal one.
\end{enumerate}
}
Unfortunately, we have no control over how the learner behaves, so even if we were able to construct such a distribution, it could be the case that for these pathological sample sizes, the learner is trying to output prices that are sufficiently larger than what it sees in the samples so that it competes better with the optimal revenue. This motivates the second main idea of our approach: we need to pick the elements and the mass of the distribution after the first $n_i$ ones so that
\begin{enumerate}
    \item when the learner draws $n_i$ samples with sufficiently high probability it will not observe anything past these elements, and,
    \item  no matter which element the learner outputs on these inputs, its revenue will be $R(n_j)$ far from the optimal one. 
\end{enumerate}
It turns out that such a goal is asking for too much: the learner can use internal randomization, so it could be the case that even upon restricting its inputs to the first $n_i$ elements of the distribution, its output cannot be always bounded. Thankfully, it turns out that we can obtain an upper bound on the output of the algorithm, conditioned on the inputs being elements from the (finite) set of the first $n_i$ elements of $\D,$ that holds \emph{with 
sufficiently high probability}. With these main conceptual ideas in mind, there are still non-trivial technical hurdles we need to overcome.

\paragraph{$o(n^{-1/2})$ rates lower bound for bounded distributions.} Next, we give an overview of the most technically involved result of our work, which is the \(o(n^{-1/2})\) lower bound
in the case of \emph{bounded} distributions. In the PAC setting, a lower bound of this sort is usually achieved by considering an ensemble of two distributions, parametrized by $n$, that are sufficiently ``close''. Then, one needs to argue that if some PAC algorithm can achieve revenue better than $\sqrt{1/n},$ it can be used to distinguish between these two distributions at a rate faster than what information theory allows for. Crucially, these two distributions change as a function of $n,$ so this approach cannot be used for our result. In fact, to obtain our result, we need to create an ensemble of uncountably many hard distributions. 
The mental picture to have in mind to understand the  construction is an infinite binary tree whose elements come from $[\nicefrac{1}{2},1]$. The distributions of the ensemble are supported on different (infinite) paths of the tree. An important ``gadget'' which will be used repeatedly (across the infinite chosen path) in the construction is a new uniform (PAC) lower bound we prove in \cref{lem:uniform-construction}, which in turn leverages a classical result about the difficulty of estimating the bias of a coin. The root of the tree has value 1. As we label new nodes along a chosen path, their values will be (weakly) decreasing. The values, even among the children of the same node, can be repeated, and can even be equal to the value of the parent. Moreover, the mass of every node across a given level is not the same. Across any path, and any node of the path, the branches indicate whether the optimal price is to the left or the right of the node. The high level idea is that, given a path, we can apply the PAC lower bound repeatedly as a ``gadget'' which takes as input certain parameters of the prefix of the path constructed so far and returns parameters for the current node (its value and mass) that create the following ``confusing'' scenario for the algorithm: after taking as input 
$k$ samples, w.h.p., the algorithm will only see samples from the prefix up to that level. Moreover, in order to obtain revenue gap smaller than $R(k)$, the algorithm must guess whether the optimal price is to the left or to the right of the current node. Thus, such an algorithm can be converted into a primitive that guesses the continuation of the path. Hence, 
in order to fool the algorithm we pick the target path in a randomized way: we decide its branches
by flipping a fair coin sequentially. Then, after we apply our gadget across the infinite chosen path, we can ensure that the algorithm will fail in this task for infinitely many sample sizes. 

There are several technical details that complicate this construction, as choosing the right parameters to make our strategy work is non-trivial and these choices turn out to be very brittle. On top of that, the argument described so far, shows that for any \emph{fixed} level $k$ guessing the continuation of the path is hard. One important subtlety is that we need to ensure this holds \emph{simultaneously} for \emph{infinitely} many such $k$. To do that, we use the reverse Fatou lemma to, in a very rough sense, ``derandomize'' the construction.

We believe our high-level idea of representing an ensemble of hard distributions as an infinite binary tree, plugging in as a ``gadget'' an appropriate PAC bound to construct the parameters of the tree, and choosing the target path in a randomized manner will be useful for obtaining universal lower bounds in other revenue maximization settings.

\paragraph{Non-Convergence of ERM algorithms.}
Since we have already developed two lower bound constructions, it is natural to see if we can
use them to prove that ERM does not converge to the optimal revenue. Unfortunately, neither of these 
approaches can be used. To see that, notice that the $o(n^{-1/2})$ lower bound is catered to bounded distributions, and \cref{Thm:Bounded} shows that ERM is indeed an optimal algorithm in this setting, so we cannot hope to show non-convergence of ERM here. Moving to our arbitrarily slow rates construction, this ensures that no matter how many samples $n$ the algorithm receives, there is a price that is larger than what the algorithm outputs and (with sufficiently high probability) achieves  revenue $R(n)$ better than the output of the algorithm.
The fact that $R(n) \rightarrow 0$ is crucial for our construction to work; indeed, if $R(n) \geq c, c > 0,$ for all~$n,$ then
our approach breaks fundamentally. 
Thus, this argument is not sufficient to ensure that the algorithm will have a constant
gap to the optimal revenue infinitely often. One attempt to overcome this issue is to use the same lower bound construction and analyze it on ERM algorithms. Unfortunately this does not work either, for a subtle but crucial reason: this construction ensures that the revenue of the prices is monotonically increasing, thus to show that an algorithm does not converge to the optimal revenue one needs to show that it will get stuck on low prices infinitely often. This is certainly \emph{not} the case for ERM; for any \emph{fixed} price $p$ of this distribution, for large enough $n,$ ERM will output a price higher than $p$ with high probability. This failed attempt reveals a crucial insight for our successful construction below: arbitrarily slow rates are caused by prices that are ``out of reach'' for the algorithm, whereas lack of convergence of ERM must be caused by large prices that look better on the sample (because of some tail event), but under the true distribution have revenue that is bounded away from the optimal by a constant. If we can ensure that these tail events happen infinitely often with constant probability, we will achieve our goal. This is precisely achieved by our construction below.
 
We design a distribution supported on the set $\{1,4^k\}_{k\in \N}$
such that the optimal monopoly price is $p^* = 1~(\rev_\D(p^*) = 1),$ the revenue of any price from $\{4^k\}_{k \in \N}$ is at most $\nicefrac{1}{2},$
yet infinitely often a learner that outputs the best price on the empirical distribution constructed by samples does not output $p^* = 1$, with constant probability. The distribution places mass (roughly) $\nicefrac{C}{4^k}$ on the value~$4^k,$ and the remaining mass is placed on $1.$ We choose the constant $C$ so that the revenue of any $p_k = 4^k$ is at most $\nicefrac{1}{2}.$ We now have all the ingredients in place to show our result. Consider the sample size $n_k = 4^k$. Then, we can see that with constant probability, the input sample will contain at least two copies with value $4^k$. Let us call this event $E_k.$ Under $E_k,$ it is not hard to see that $p_k = 4^k$ has better revenue guarantee than the optimal price $p^* = 1$
on the empirical distribution; indeed, since there are (at least) two copies of $4^k$, its empirical revenue is (at least) 2,
whereas the empirical revenue of $p^*$ is always 1. Since this happens with constant probability for all $k,$ we have witnessed an infinite sequence of sample sizes where the revenue of the output of ERM is bounded away by a constant from the optimal revenue. Interestingly, the support of this hard distribution is discrete and closed.

\paragraph{Exponential Rates Lower Bound.} Lastly, we prove a fourth lower bound, showing that no learner can achieve rates faster than exponential even for distributions supported only on three values\footnote{The formal result holds even for two values.} $v_1 = 1, v_2 = 2, v_3 = 3$. This result follows from standard arguments:
consider two distributions $\D_1, \D_2$ such that $\D_1$
splits its mass between $v_1, v_2$ and $\D_2$ splits its mass between $v_2, v_3.$ Notice that the optimal price for $\D_1$ is $p_1 = 2$ and the optimal price for $\D_2$ is $p_2 = 3$. 
No matter which of the two distributions the adversary chooses, with exponentially small probability, a draw of $n$ samples will
contain only $v_2$. Thus, under this event, the learner must make an arbitrary choice between $p_1, p_2.$ A pigeonhole principle argument shows that no matter how the algorithm behaves, its choice will be wrong for one of the two distributions infinitely often, under this event.

\subsubsection{Algorithmic Constructions and Analyses}
{Our algorithms rely on the following general recipe:
\begin{enumerate}
    \item First, we draw a sample $S \sim \mathcal D^n$ of valuations and construct the empirical distribution $\wh {\mathcal D}_n$.

    \item Then, depending on the type of the support of the valuation distribution $\mathcal D$, the algorithm performs some variant of the ERM principle with respect to $\wh {\mathcal D}_n$.
\end{enumerate}

\paragraph{Finite Support and ERM.} If the support is finite $\{p_1,...,p_m\}$, then the algorithm outputs the observation from the training set $S$ that maximizes the empirical expected revenue with respect to $\wh{\mathcal D}_n$ (observe that the optimal revenue is realized by at least one of the points in the support.) To see why this approach achieves exponential rates, first notice that since there are finitely many prices, the revenue gap between the best and the ``second best'' price is constant.\footnote{The formal version of this argument works even when the set of optimal prices is a multiset or the set of ``second-best'' prices is empty.} Then, we can use the CDF concentration inequality (\cref{lem:dkw-ineq}) and the fact
that the prices are bounded (since the support is finite), to argue that, except for an event with exponentially small probability, the empirical revenue of the optimal price will outperform the empirical revenue of any other price. Recall that exponential rates convergence means that for any distribution, the revenue gap drops as~$C \cdot e^{-c \cdot n}, $ for distribution-dependent $c, C >0.$ In this case, we can show that these constants depend on two parameters of the distribution: the revenue gap between the best, and second-best price, as well as the magnitude of the largest price in the support.


\paragraph{Closed-Discrete Support and Structural ERM.} 
{The natural next after studying distributions with finite support is to consider distributions with an (unbounded) discrete support, such as the natural numbers. To be concrete, we will consider the case where
the support of the valuation distribution is closed and discrete, i.e., its elements are isolated (not arbitrarily close to each other) and it has no limit points. 
Since the support is not finite, our argument on the performance of ERM from the previous paragraph does not apply. In fact, our lower bound shows that ERM provably does not even converge to the optimal revenue in this setting (see \Cref{Thm:Discrete}[Item 2]). But if ERM achieves exponential rates for finitely supported distributions, why does it behave so poorly for discrete, unbounded distributions?} 

To understand that, it is instructive to explain which ingredient of the exponential rates upper bound is violated in our hard instance. The only two properties we used for the exponential rates bound are that \textbf{i)} the revenue gap between the optimal and ``second-optimal'' price is constant, and \textbf{ii)} the prices are bounded. In our lower bound construction, property \textbf{(i)} is not violated; indeed, it still the case that the optimal price has a constant revenue gap to the second best. However, it is clearly the case that we have no a-priori bound on the optimal price. It is also not hard to see that any algorithm which tries to artificially bound the space of prices it is searching over is doomed to fail. Thus, we need to answer the following: how can an algorithm search over increasing large price spaces, avoid getting fooled by tail events, and pinpoint the optimal price (almost) exponentially fast?

To do that, we design a variant of ERM that is \textbf{i)} biased towards smaller prices to protect against tail events, while \textbf{ii)} continuously decreasing the bias over time to ensure that it does not get stuck on prohibitively low prices. As we explained above, both of these components are necessary to achieve the desired result. 
The algorithm uses the training set $S = \{p_1,...,p_n\}$ to create the empirical distribution and sorts the
observed points in non-decreasing order $p_1 \leq ... \leq p_n$. It then finds the rightmost point of the sorted list that beats all its
prior points in empirical revenue with some confidence ``$\mathrm{gap}(n)$'', where this $\mathrm{gap(n)}$ is decreasing as a function of the sample size. More formally, the algorithm outputs the price $p_{i^*} \in S$ such that
\[
i^\star = \max_{i \in [n]} \{ ~\forall j \leq i ~:~\mathrm{Rev}_{\wh{\mathcal{D}}_n}(p_i) >
\mathrm{Rev}_{\wh{\mathcal{D}}_n}(p_j) + \mathrm{gap}(n, p_i,p_j)
\}
\]
It turns out there exists a choice of the gap function to balance these two desiderata, i.e, have continuously increasing exploration while protecting against tail events, that allows us to get the (almost) exponential rates result.

\paragraph{Unbounded Support and Capped ERM.}
Our third algorithmic variant occurs in the case where the support of the valuation distribution is unbounded. 
In this regime, we show two results: \textbf{i)} under no assumptions on the underlying distribution (even allowing for infinite revenue), there is an algorithm whose empirical revenue converges to the optimal one, i.e., $\lim_{n\rightarrow\infty} \rev_\D(t_n) = \opt_\D$,\footnote{If $\opt_\D = \infty$, then $\lim_{n\rightarrow\infty} \rev_\D(t_n) = \infty.$} and \textbf{ii)} if the optimal revenue of $\D$ is achieved by some price $p^*$ then there is an algorithm that achieves (almost) $\nicefrac{1}{\sqrt{n}}$ rates.\footnote{Recall that if the optimal revenue is not achieved by some price, we have an arbitrarily slow rates lower bound.}
Our algorithm for this case is inspired by the \emph{guarded} ERM approach of \citet{dhangwatnotai2010revenue,huang2015making}, i.e.,
for any sample size $n$, we output the best distribution on the sample that is bounded by a certain number.
The main difference with these works is that our cut-off point is a ``hard'' threshold that does not depend on the realization of the samples, but is rather a (slowly growing) function of $n$, whereas they choose the optimal price on the empirical distribution excluding the top $\varepsilon \cdot n$ fraction of the prices, for some appropriate $\varepsilon$. 
The reason for using this thresholding approach is that since we do not know which price achieves the optimal revenue, we explore an increasing space of the support, i.e., we perform ERM on the set $G_n = [0, g(n)]$ for some increasing function $g$. It turns out that this type of truncation is necessary; indeed, as we showed, unrestricted ERM cannot converge to the optimal revenue. Instantiating $g(\cdot)$ with different functions allows us to get the desired rates.
}

\paragraph{Analysis of ERM for Bounded Distributions.} In the setting of distributions with bounded support, our result shows that ERM is indeed an optimal learner, and the rate it achieves is (slightly) better than its uniform counterpart. To get this result, we use a ``localized'' version of Bernstein's inequality, to show that the concentration of the revenue of all $\varepsilon$-optimal prices happens at a rate (roughly) $O\left(\sqrt{\nicefrac{\Delta(\varepsilon)}{n} \log(1/\Delta(\varepsilon)) }\right),$ where $\Delta(\varepsilon)$ should be thought of as a variance term. Our main technical work here comes in proving that for all distributions $\Delta(\varepsilon) = o(1)$. This, combined with our concentration inequality, gives the desired $o(1/\sqrt{n})$ bound.


\subsection{Takeaways and Open Questions}\label{sec:conclusion}
We conclude by highlighting a few takeaways from our work. First, unlike in the PAC setting—where obtaining any meaningful learnability guarantees typically requires strong distributional assumptions (e.g., regularity, MHR, or bounded support)—in our learning-curves framework the problem is always tractable. Second, our results show that whether the optimal revenue is actually attained by some price or only approached asymptotically has a dramatic effect on the difficulty of the learning problem. Third, the landscape of optimal learning rates is surprisingly rich: the geometry of the distribution’s support plays a central role in determining the taxonomy of achievable rates. Finally, our framework is what makes it possible to compare natural learning primitives such as structured ERM and ERM. While in the PAC setting both have trivial worst-case guarantees—even for discrete, closed supports—in our framework structured ERM can substantially outperform ERM, and in fact we show that the latter does not even converge to the optimal revenue.

There are several immediate directions for subsequent work. In particular, many mechanism-design learning problems for which PAC-style bounds are known can now be revisited through our framework. Below we list some concrete directions.

\paragraph{Rates For General Single-Dimensional Setting} In our work we studied the problem of a selling a single item to a single buyer. We believe our results, and new ideas, can be utilized to study a general single-dimensional setting of selling one item to multiple buyers.



\paragraph{Rates Characterization for Distribution Families} Instead of fixing the support of the class of distributions, it could be interesting to study parametric classes of distributions, including regular and MHR distributions.
For instance, our results already show that the universal rate for regular distributions is faster than the PAC rate. However, it would be interesting to obtain a tight characterization.

\paragraph{Rates of Empirical Revenue Maximization}
Our results show a sharp contrast in the behavior of ERM across distribution classes. On the one hand, \Cref{Thm:Bounded} shows that ERM is optimal for distributions with bounded support. On the other hand, \Cref{Thm:Discrete} shows that ERM can fail to be Bayes-consistent for closed and discrete unbounded supports, even though a structured ERM algorithm achieves nearly exponential rates in that setting. Since ERM is a fundamental learning primitive, it remains interesting to characterize the distributional conditions under which ERM succeeds and the rates it achieves.


}

\subsection{Related Work}
\label{sec:RelatedWork}

\paragraph{Revenue Maximization. } {The study of revenue maximization from 
samples through a formal learning theoretic model was initiated in the seminal work of \citet{cole2014sample}, building on earlier work by \citet{balcan2008reducing}. PAC sample complexity bounds for regular and MHR distributions
were shown by \citet{dhangwatnotai2010revenue},
while learning finitely supported distributions
was also studied by \citet{elkind2007designing}, without 
providing bounds on the number of samples needed.
\citet{huang2015making} showed (almost) tight PAC sample complexity bounds in the single-item single-buyer setting.
A long line of work \citep{morgenstern2016learning,morgenstern2015pseudo,gonczarowski2017efficient,syrgkanis2017sample,devanur2016sample} provided improved sample complexity bounds for single parameter environments, culminating in the work of \citet{guo2019settling} who provided optimal sample complexity bounds (up to polylogarithmic factors) using efficient algorithms. The sample complexity of auctions has also been studied in the more challenging multi-parameter setting \citep{balcan2016sample,cai2017learning,balcan2018general,balcan2018dispersion,brustle2020multi,gonczarowski2021sample}.
A different line of work \citep{neeman2003effectiveness, segal2003optimal,baliga2003market} showed that certain classes
of valuation distributions, such as those with finite first moment, can be learnt asymptotically, without providing
bounds on the rates of convergence.
{Relatedly, \citet{alon2017submultiplicative} 
studied when the estimated revenue from samples of every price converges to its true revenue, providing bounds that have the
same flavor as the concentration bounds for the CDF of a distribution.
}
}

\paragraph{Learning Curves.}
The roots of universal learning go back to the work of \cite{stone1977consistent} for establishing Bayes consistency for the nearest-neighbors algorithm. Regarding universal lower bounds, there is a rich line of works for various problems such as classification, regression and density estimation \citep{boyd1978lower,devroye1984distribution,antos2000lower,antos1999performance,devroye1983arbitrarily,devroye2002distribution,antos2001convergence}. More recently, \citet{bousquet2021theory} provided a clean framework to study universal rates and gave 
a complete characterization of the optimal rates for binary classifications. Subsequent works studied other classification settings such as multiclass classification \citep{kalavasis2022multiclass,hanneke2023universal}, and active learning
\citep{hanneke2022universal,hannekeuniversal2024b}. Furthermore, \citet{hanneke2024universal} studied the optimal universal rates of Empirical Risk Minimization,  \citet{attias2024universal} studied rates of regression, and \citet{kalavasis2024limits}  provided rates for language generation. \citet{bousquet2023fine} provided more fine-grained
bounds for the rates of binary classification.

\section{Roadmap}
In \cref{sec:formal-defs} we give some more formal definitions from our framework. In \cref{proof:BayesConsistency} we give the formal proof of the result regarding Bayes consistency. In \cref{proof:universalDiscreteLB} we prove the arbitrarily slow rates lower bound. In \cref{proof:upperUnbounded} we show the $\omega(1/\sqrt{n})$ upper bound for distributions whose optimal revenue is achieved by some price. In \cref{proof:root-n-lb} we provide the proof
of the $o(1/\sqrt{n})$ lower bound. In \cref{proof:upperBounded} we give the proof of the $o(1/\sqrt{n})$ upper bound. In \cref{proof:Exponential}
we give the proof of our (nearly) exponential rates upper bound for closed and discrete support. In \cref{sec:exp-rates-upper-bound-finite} we discuss the exponential rates upper bound algorithm for finite support.
In \cref{sec:erm-proof} we give the proof of the ERM lower bound. In \cref{proof:expr-rates-lower-bound} we give the proof of the exponential rates lower bound.

\section{Formal Definitions of Learning Rates}\label{sec:formal-defs}

We have already defined in Equation \eqref{eq:PAC-UB} and \Cref{def:universal} the learning rates in the  PAC setting and the universal case respectively. We now introduce formal definitions for the associated lower bounds.

\subsection{PAC Lower Bound}
We start with the definition of a PAC lower bound.

\begin{defn}[PAC Lower Bound]
Let \(R:\N\to(0,1]\) be a rate function. We say that the class of valuation distributions \(\mathbb D\) is not PAC learnable at a rate faster than \(R\) if, for every learning algorithm \((t_n)_{n\in\N}\), there exist constants \(C,c>0\) such that
\begin{equation}
\limsup_{n \to \infty} \sup_{\D \in \mathbb D}
\frac{\epsilon_n(t_n,\D)}{R(cn)} \ge C.
\label{eq:PAC-LB}
\end{equation}
\end{defn}

In words, for every algorithm \((t_n)_{n\in\N}\), there exist infinitely many sample sizes \(n_k\) and distributions \(\D_k\in\mathbb D\) such that
\[
\epsilon_{n_k}(t_{n_k},\D_k) \ge C R(c n_k)
\]
for some constants \(C,c>0\).

\subsection{Universal Lower Bound}
We now introduce the notion of a universal, or individual, lower bound.

\begin{defn}[Universal Rates -- Lower Bound]
Let \(R:\N\to(0,1]\) be a rate function, i.e., \(R(n)\downarrow 0\). We say that a class of valuation distributions \(\mathbb D\) is not universally learnable at a rate faster than \(R\) if, for every learning algorithm \((t_n)_{n\in\N}\), there exist a distribution \(\D\in\mathbb D\) and constants \(C,c>0\) such that
\begin{equation}
\limsup_{n \to \infty}
\frac{\epsilon_n(t_n,\D)}{R(c n)} \ge C.
\label{eq:LearningCurve-LB}
\end{equation}
\end{defn}

{
\subsection{Arbitrarily Slow Universal Rates}
Using the previous lower bound definition, and following \citet{bousquet2021theory}, we define what it means for $\D$ to
require \emph{arbitrarily slow rates.}

\begin{defn}
[Arbitrarily Slow Rates]
\label{def:slow}
We will say that a class of 
valuation distributions $\mathbb D$ requires arbitrarily slow rates if 
for any rate function $R,$ it cannot be universally learned at a rate faster than $R.$
\end{defn}
In words, this means that for any learning algorithm $(t_n)_n$ and
for any rate function $R,$ there exists some $\D \in \mathbb{D}$
for which the revenue gap of the algorithm is at least $R(n),$
for infinitely many $n \in \N.$
Thus, this rate captures problems that are extremely ``hard:'' if a class $\mathbb D$ requires arbitrarily slow rates we cannot find any upper bound on the rate of decay of its learning curves.

\subsection{Optimal Universal Rates}
Combining the upper bound and lower bound definitions, we give 
the following definition about the optimality of some rate.

\begin{defn}[Optimal Rate]\label{def:opt-rate}
    We say that a class of valuation distributions $\mathbb{D}$
    is universally learnable at an optimal rate $R$, if \textbf{i)} 
    $\mathbb{D}$ is universally learnable at rate $R$, and \textbf{ii)} 
    $\mathbb{D}$ is not universally learnable at rate faster than $R.$
\end{defn}

\subsection{Universal $o(R(n))$ Rates} 
A function $f(n)$ is $o(R(n))$ if $\lim_{n \to \infty} f(n)/R(n) = 0.$ We will now need the following definition of $o(R(n))$ rates.

\begin{defn}
[Learnable at Universal Rate $o(R(n))$]
\label{def:o-R-upper}
We say that a class of valuation distributions $\mathbb{D}$
    is universally learnable at an $o(R)$ rate
    if there is an algorithm $(t_n)_{n \in \N}$
    such that for any $\D \in \mathbb{D}$ there exist $C,c>0$ such that $\epsilon_n(t_n, \D) \leq g(n)$ for  \emph{some}
    rate $g(n) \in o(R(n)).$ 
\end{defn}

Similarly, we need to introduce a universal $o(R(n))$ lower bound.

\begin{defn}
[Universal $o(R(n))$ Lower Bound]
\label{def:o-R-lower}
We say that a class of valuation distributions $\mathbb{D}$
    has a universal $o(R)$  lower bound if 
    for every rate
    $g(n) \in o(R(n))$ the class $\mathbb{D}$ is not universally learnable 
    at rate faster than $g(n).$
\end{defn}

Combining the above definitions, we can define optimal $o(R(n))$ rates.

\begin{defn}[Optimal $o(R(n))$]
\label{def:o-R-opt}
We say that a class of valuation distributions \(\mathbb D\) is universally learnable at an optimal \(o(R)\) rate if:
\begin{enumerate}
    \item \(\mathbb D\) is universally learnable at an \(o(R)\) rate, and
    \item for every rate \(g(n)\in o(R(n))\), the class \(\mathbb D\) is not universally learnable at a rate faster than \(g(n)\).
\end{enumerate}
\end{defn}

\subsection{Universal $e^{-o(n)}$ Rates}

We use \(e^{-o(n)}\) to denote the family of rates of the form \(e^{-h(n)}\), where \(h:\N\to\R_+\) satisfies
\[
h(n)\to\infty
\qquad\text{and}\qquad
h(n)=o(n).
\]
Thus these rates converge to zero, but more slowly than any exponential rate \(e^{-cn}\) with constant \(c>0\). Examples include \(e^{-\sqrt n}\), \(e^{-\log n}\), and \(e^{-n/\log n}\).

\begin{defn}[Learnable at Universal Rate $e^{-o(n)}$]
\label{sec:exp-o-upper}
We say that a class of valuation distributions \(\mathbb D\) is universally learnable at rate \(e^{-o(n)}\) if, for every nondecreasing function \(g:\N\to\R_+\) such that \(g(n)\to\infty\) and \(g(n)=o(n)\), the class \(\mathbb D\) is universally learnable at rate \(e^{-g(n)}\).
\end{defn}

\subsection{Universal $\omega(R(n))$ Rates}

In a similar fashion we can define what it means for an algorithm to achieve $\omega(R(n))$ rates such as $\omega(1/\sqrt{n})$. Intuitively, it means that the class is learnable at a rate slightly slower than $R(n)$: A class is universally learnable at rate $\omega(R(n))$ if for any $g(n) \in \omega(R(n))$, there is an algorithm that achieves this rate.

\begin{defn}
[Learnable at Universal Rate $\omega(R(n))$]
\label{def:omega-R-opt}
We say that a class of valuation distributions $\mathbb{D}$
    is universally learnable at  $\omega(R(n))$ rate
    if for every $g(n) \in \omega(R(n))$ rate, $\mathbb{D}$ 
    is universally learnable 
    at rate $g(n).$ 
\end{defn}

We note that the collection of rates $\omega(e^{-n})$ strictly contains the collection $e^{-o(n)}$ and this why we included two distinct definitions.
To see the strict inclusion,  the class  \( \omega(e^{-n}) \) contains functions \( f(n) \) that grow asymptotically faster than \( e^{-n} \), meaning:
   \[
   \lim_{n \to \infty} \frac{f(n)}{e^{-n}} = \infty.
   \]
   This includes: (i) sub-exponential decays (like \( e^{-o(n)} \)), such as \( e^{-\sqrt{n}} \), and
   (ii) exponential decays with slightly slower rates, such as \( n e^{-n} \). These have exponents of the form \( -g(n) \) where \( g(n) \in \Theta(n) \) but \( n - g(n) \to \infty \).
   
}

\section{General Distributional Setting}
In this section we consider a general revenue maximization setting where we do not place any restrictions on the distribution.

\subsection{Bayes Consistency}
\label{proof:BayesConsistency}
In this section we give the formal proof of the result regarding Bayes-consistent learners.

\begin{proof}
[Proof of \Cref{thm:single-item-bayes-consistency}]
    Let $\D$ be the underlying valuation 
    distribution and $\opt_\D = \sup_{p \in \R_+} p \cdot \Pr_{v \sim \D}[v \geq p].$
    We define the learning algorithm as follows:
    for every $n \in \N$ we create
    the empirical distribution $\wh \D_n$ on the samples
    and output the price $\wh t_n =\argmax_{p \leq \log n} p \cdot \Pr_{v \sim \wh \D_n} [v \geq p].$
    Let $F_n$ be the empirical CDF 
    constructed by taking $n$ i.i.d. samples
    from $\D$.
    We consider two cases.\\

    \noindent \textbf{Case $i$: $\opt_\D = \infty.$}  Consider any number $M \in \R_{+}.$ Then,
    by definition of $\opt_\D,$ there is 
    some $p_M \in \R_+$ such that
    $\rev_\D(p) \geq M.$ Consider 
    some $n_0 \in \N$ large enough such that
     $\log(n_0) \geq p_M$ and $2e^{-2n_0 \cdot \left(\frac{M}{4 \log n_0}\right)^2} \leq \nicefrac{1}{2}.$ Then, notice that
     for all $n \geq n_0$ we also have
      $\log(n) \geq p_M$ and $2e^{-2n \cdot \left(\frac{M}{4 \log n}\right)^2} \leq \nicefrac{1}{2}.$ Recall the DKW inequality (\cref{lem:dkw-ineq}) which
      shows that for all $n \in \N, \varepsilon > 0$
      it holds that
      \[
         \Pr[\sup_{x \in \R}\abs{F_n(x) - F(x)} > \varepsilon] \leq 2e^{-2n\varepsilon^2} \,.
      \]
      Substituting $\varepsilon = M / (4\log n)$
      gives
      \[
  \Pr\left[\sup_{x \in \R}\abs{F_n(x) - F(x)} > \frac{M}{4\log n}\right]
  \leq
  2e^{-2n \left(\frac{M}{4\log n}\right)^2}.
\]
            Thus, for all \(n \geq n_0\),
      \[
         \Pr\left[\sup_{x \in \R}\abs{F_n(x) - F(x)} > \frac{M}{4\log n}\right] \leq \frac{1}{2}.
      \]
      Define the good event
      \[
      E_n :=
      \left\{
      \sup_{x \in \R}\abs{F_n(x)-F(x)}
      \le
      \frac{M}{4\log n}
      \right\}.
      \]
      Then \(\Pr(E_n)\ge 1/2\). Let \(\wh t_n\)
      be the price our algorithm outputs.
      Under the event \(E_n\), we have
      \[
        \abs{F_n(\wh t_n) - F(\wh t_n)} \leq \frac{M}{4\log n} \,,
      \]
      so since $\wh t_n \leq \log n$ it follows that
      \[
  \abs{\wh t_n \cdot (1-F_n(\wh t_n)) - \wh t_n \cdot (1-F(\wh t_n))}
  \leq
  \frac{\wh t_n \cdot M}{4\log n}
  \leq
  \frac{M}{4}.
\]
      Similarly,
      \[
        \abs{ p_M \cdot (1-F_n(p_M)) -  p_M \cdot (1- F( p_M))} \leq \frac{p_M \cdot M}{4\log n} \leq \frac{M}{4} \,.
      \]
      By definition of the algorithm, it holds
      that
      \[
        \wh t_n \cdot (1- F_n(\wh t_n)) \geq p_M \cdot (1 - F_n(p_M)) \,.
      \]
            Chaining the above inequalities, we get that on the event \(E_n\),
      \begin{equation}\label{eq:bayes-rev-bound}
          \rev_\D(\wh t_n)
          =
          \wh t_n \cdot (1- F(\wh t_n))
          \geq
          p_M \cdot (1- F(p_M)) - \frac{M}{2}
          \geq
          \frac{M}{2}.
      \end{equation}
      Therefore,
      \[
      \E[\rev_\D(\wh t_n)\mid E_n]\ge \frac{M}{2}.
      \]
        
      Thus, for $n \geq n_0$ the revenue of the algorithm satisfies
      \begin{align*}
                  \E[\rev_\D(\wh t_n)] &= \E[\rev_\D(\wh t_n)| E_n] \cdot \Pr[E_n] + \E[\rev_\D(\wh t_n)| \lnot E_n] \cdot \Pr[\lnot E_n] & (\text{Total expectation})\\
                  &\geq \E[\rev_\D(\wh t_n)| E_n] \cdot \Pr[E_n]  & (\text{Non-negativity})\\
                  &\geq  \frac{M}{2} \cdot \Pr[E_n]  & (\text{\Cref{eq:bayes-rev-bound}}) \\
                   &\geq  \frac{M}{4}  \,.& (\text{Def. of $E_n$}) 
      \end{align*}
      Since this holds for any $M > 0,$ we 
      have shown that $\lim_{n \rightarrow \infty} \E[\rev_\D(\wh t_n)] = \infty = \opt_\D.$\\

       \noindent \textbf{Case $ii$: $\opt_\D < \infty.$} 
        If \(\opt_\D=0\), then \(\rev_\D(p)=0\) for every \(p\), and therefore
       \[
       \E[\rev_\D(\wh t_n)] = 0 = \opt_\D
       \]
       for all \(n\). Hence the claim is immediate. We may therefore assume that \(\opt_\D>0\).

       This case is almost identical to the previous one, but we spell out the details for completeness. For any $\varepsilon > 0,$ let $p_\varepsilon \in \R_+$ be a price such that
       $\rev_\D(p_\varepsilon) \geq \opt - \varepsilon.$
       Notice that such a price always exists.
       Consider 
    some $n_0 \in \N$ large enough such that
     $\log(n_0) \geq p_\varepsilon$ and $2e^{-2n_0 \cdot \left(\frac{\varepsilon}{4 \log n_0}\right)^2} \leq \nicefrac{\varepsilon}{\opt_\D}.$ Then, notice that
     for all $n \geq n_0$ we also have
      $\log(n) \geq p_\varepsilon$ and $2e^{-2n \cdot \left(\frac{\varepsilon}{4 \log n}\right)^2} \leq \nicefrac{\varepsilon}{\opt_\D}.$ Recall the DKW inequality (\cref{lem:dkw-ineq}) which
      shows that for all $n \in \N, \tilde \varepsilon > 0$
      it holds that
      \[
         \Pr[\sup_{x \in \R}\abs{F_n(x) - F(x)} > \tilde \varepsilon] \leq 2e^{-2n\tilde \varepsilon^2} \,.
      \]
      Substituting $\tilde \varepsilon = \varepsilon / (4\log n)$
      gives
      \[
  \Pr\left[\sup_{x \in \R}\abs{F_n(x) - F(x)} > \frac{\varepsilon}{4\log n}\right]
  \leq
  2e^{-2n \left(\frac{\varepsilon}{4\log n}\right)^2}.
\]
           Thus, for all \(n \geq n_0\),
      \[
         \Pr\left[\sup_{x \in \R}\abs{F_n(x) - F(x)} > \frac{\varepsilon}{4\log n}\right]
         \leq
         \frac{\varepsilon}{\opt_\D}.
      \]
      Define the good event
      \[
      E_n :=
      \left\{
      \sup_{x \in \R}\abs{F_n(x)-F(x)}
      \le
      \frac{\varepsilon}{4\log n}
      \right\}.
      \]
      Then \(\Pr(E_n)\ge 1-\varepsilon/\opt_\D\). Let \(\wh t_n\)
      be the price our algorithm outputs.
      Under the event \(E_n\), we have
      \[
        \abs{F_n(\wh t_n) - F(\wh t_n)} \leq \frac{\varepsilon}{4\log n} \,,
      \]
      so since $\wh t_n \leq \log n$ it follows that
      \[
  \abs{\wh t_n \cdot (1-F_n(\wh t_n)) - \wh t_n \cdot (1-F(\wh t_n))}
  \leq
  \frac{\wh t_n \cdot \varepsilon}{4\log n}
  \leq
  \frac{\varepsilon}{4}.
\]
      Similarly,
      \[
        \abs{ p_\varepsilon \cdot (1-F_n(p_\varepsilon)) -  p_\varepsilon \cdot (1- F( p_\varepsilon))} \leq \frac{p_\varepsilon \cdot \varepsilon}{4\log n} \leq \frac{\varepsilon}{4} \,.
      \]
      By definition of the algorithm, it holds
      that
      \[
        \wh t_n \cdot (1- F_n(\wh t_n)) \geq p_\varepsilon \cdot (1 - F_n(p_\varepsilon)) \,.
      \]
            Chaining the above inequalities, we get that on the event \(E_n\),
      \begin{equation}\label{eq:bayes-rev-bound-2}
          \rev_\D(\wh t_n)
          =
          \wh t_n \cdot (1- F(\wh t_n))
          \geq
          p_\varepsilon \cdot (1- F(p_\varepsilon)) - \frac{\varepsilon}{2}
          \geq
          \opt_\D -\frac{3\varepsilon}{2}.
      \end{equation}
      Therefore,
      \[
      \E[\rev_\D(\wh t_n)\mid E_n]\ge \opt_\D-\frac{3\varepsilon}{2}.
      \]
        
      Thus, for $n \geq n_0$ the revenue of the algorithm satisfies
      \begin{align*}
                  \E[\rev_\D(\wh t_n)] &= \E[\rev_\D(\wh t_n)| E_n] \cdot \Pr[E_n] + \E[\rev_\D(\wh t_n)| \lnot E_n] \cdot \Pr[\lnot E_n] & (\text{Total expectation})\\
                  &\geq \E[\rev_\D(\wh t_n)| E_n] \cdot \Pr[E_n]  & (\text{Non-negativity})\\
                  &\geq  \left(\opt -\frac{3\cdot \varepsilon}{2} \right)\cdot \Pr[E_n]  & (\text{\Cref{eq:bayes-rev-bound-2}}) \\
                   &\geq   \left(\opt - \frac{3\varepsilon}{2}\right) \cdot \left(1-\frac{\varepsilon}{\opt}\right) & (\text{Def. of $E_n$}) \\
                   &\geq  \opt - \frac{5\varepsilon}{2}  \,. & 
      \end{align*}
      Since this holds for any $\varepsilon > 0,$
      and $\E[\rev_\D(\wh t_n)] \leq \opt,$ we 
      have shown that $\lim_{n \rightarrow \infty} \E[\rev_\D(\wh t_n)] = \opt.$
\end{proof}

\subsection{Arbitrarily Slow Rates}
\label{proof:universalDiscreteLB}
In this section, we will prove the following.
\begin{thm}
[Arbitrarily Slow Rates]
\label{thm:arbitrarily-slow}
Let $\mathbb{D}$ be the class of valuation distributions $\mathcal{D}$ that are supported on a discrete closed set and such that $\opt_\D < \infty$. Then, 
$\mathbb{D}$ requires \emph{arbitrarily slow} rates.   
\end{thm}

\begin{proof}
[Proof of \Cref{thm:arbitrarily-slow}]
    Let $\phi: \N \rightarrow (0,1)$ be a rate function, i.e., $\lim_{n \rightarrow \infty} \phi(n) = 0.$ 
    We define the function $R:\N \rightarrow \R$
    inductively as follows. We let $R(1) = \phi(1)$
    and $R(j) = \phi(j),$ if $\phi(j) \leq R(j-1)$, 
    otherwise $R(j) = R(j-1)$, for any $j \geq 2.$ 
    Notice that $R(\cdot)$ is non-increasing
    and there are infinitely many $i \in \N$
    for which $R(i) = \phi(i).$
    Recall that
    we denote by $\{t_n : \R^n_{+} \rightarrow \R_+\}$ the learning algorithm maps $n$ valuations to a price for any $n \in \N$.

    \paragraph{Construction of the Hard Distribution}
    Then, we define two sequences of non-negative numbers $(i_j)_{j \in \N}, (P_j)_{j \in \N}$
    inductively in the following way: we start with 
    $i_1 = 0, P_1 = 1.$ Then, for every $j \in \N, j \geq 2,$ assuming we have already defined
    $i_1,\ldots,i_{j-1}, P_1,\ldots,P_{j-1}$
    we first define the number $c_{j-1}$
    as follows: we consider all the possible
    $k_{j-1} = (j-1)^{j-1}$ ordered datasets of 
    size $j-1$ whose elements come from the set
    $\{i_1,\ldots,i_{j-1}\}$ the algorithm can get 
    as input. Let $S_{j-1}^1, S_{j-1}^2,\ldots,S_{j-1}^{k_{j-1}}$ be all these different datasets. 
    For any $\ell \in [k_{j-1}]$ we let
    $c_{j-1}^\ell$ be an arbitrary (finite) number in
    $\R_+$ such that
    \[
        \Pr\left[ t_{j-1}\left(S_{j-1}^\ell\right) > c_{j-1}^\ell \right] \leq \frac{R(j-1)}{4} \,,
    \]
    where the probability is taken
    with respect to the internal randomness
    of the algorithm. 
    Notice that this number is always well-defined.
    We let 
    \[
        c_{j-1} = \max_{\ell \in [k_{j-1}]}\left\{c_{j-1}^\ell\right\} \,.
    \]
        We then define \(i_j, P_j\) to be any finite numbers in \(\N\) and \((0,1)\), respectively, so that the following constraints are satisfied:
    \begin{align*}
        i_j &> \max\left\{i_{j-1}, c_{j-1}\right\},\\
        P_j &\leq \min\left\{\frac{P_{j-1}}{2}, \frac{R(j-1)}{2(j-1)}\right\},\\
        i_j \cdot P_j &= 2 - R(j-1).
    \end{align*}
    Notice that the above set of constraints
    is non-empty (as we can choose $i_j$ sufficiently large and $P_j$ sufficiently small, as needed), and for any feasible $P_j$ it
    holds that $P_j > 0.$ To see that, 
    notice that $2- R(j-1) \geq 1$
    and any feasible $i_j$ satisfies $i_j > 0,$
    thus since $i_j \cdot P_j = 2- R(j-1) \geq 1$
    it follows that $P_j > 0.$

    Having defined $(i_j)_{j \in \N}, (P_j)_{j \in \N}$ we define a distribution $\D$ supported
    entirely on $(i_j)_{j \in \N}$ where
    \[
        \Pr_{X \sim \D}[X = i_j] = P_j - P_{j+1} \,.
    \]
    To verify that this is indeed a valid 
    distribution notice that
    \[
         P_j - P_{j+1} \geq P_j - P_j /2 = P_j / 2 > 0\,,
    \]
    and
    \[
        \sum_{j=1}^\infty \left(P_j - P_{j+1}\right) = P_1 - \lim_{j \rightarrow \infty} P_j = 1\,,
    \]
    since
    \[ 
       0<  P_j \leq \frac{R(j-1)}{2(j-1)} \,,
    \]
    so $\lim_{j \rightarrow \infty} P_j = 0.$
    
    \paragraph{Optimal Revenue under $\D$}
    For any $j \in \N, j \geq 2$, notice that
    \[
        \rev_{\D}(i_j) = i_j \cdot \Pr_{X \sim \D}[X \geq i_j] = i_j \cdot \sum_{j' \geq j} \left(P_{j'} - P_{j'+1}\right)= i_j \cdot P_j = 2 - R(j-1) \,.
    \]
    Moreover, for any number $p \in \R_+, p \notin \supp(\D)$, let $j' = \min\{j \in \N : i_j > p\}$. It holds that
    \[
        \rev_{\D}(p) = p \cdot \Pr_{X \sim \D}[X \geq p] = p \cdot \Pr_{X \sim \D}[X \geq i_{j'}] < i_{j'} \cdot \Pr_{X \sim \D}[X \geq i_{j'}] = 2 - R(j'-1) \,.
    \]
    Thus, we can see that
    \[
        \opt_{\D} = \sup_{p \in \R_+} \rev_{\D}(p) = 2\,.
    \]

    \paragraph{Sub-Optimality of $t_n$ under $\D$}
    Notice that, by definition of $\D$, 
    for any $j \in \N$
    it holds that
    \[
        \Pr_{X \sim \D}[X \leq i_j] = 1 - \Pr_{X \sim \D}[X \geq i_{j+1}] = 1 - P_{j+1} \geq 1 - \frac{R(j)}{2j} \,.
    \]
    By taking a union bound over $j$ i.i.d. draws
    from $\D$ we see that
    \[
        \Pr_{X_1,\ldots,X_j \sim \D^j}[X_1 \leq i_j ,\ldots,X_j \leq i_j] \geq 1-\frac{R(j)}{2} \,.
    \]

    Let $E_j$ be the event that
    $X_1 \leq i_j, \ldots, X_j \leq i_j.$
    Then, under the event $E_j$, for
    any realization of $X_1,\ldots,X_j$, by the definition of 
    $i_{j+1}$, we have
    that
    \[
        \Pr[t_j(X_1,\ldots,X_j) \geq i_{j+1} | E_j] \leq \frac{R(j)}{4} \,,
    \]
    where the probability is taken with respect to the random
    draw of the algorithm.    Let $E_j'$ be the event that $t_j(X_1,\ldots,X_j) \leq i_{j+1}$.
    Next, we show that for
    any $p \in [0, i_{j+1}]$
    it holds that
    \[
        \rev_{\D}(p) \leq 2 - R(j) \,.
    \]
        To see that, first note that the case \(p=0\) is immediate, since \(\rev_\D(0)=0\le 2-R(j)\). Now fix \(p\in(0,i_{j+1}]\), and let
    \(j' \in \N\) be the smallest number such that \(i_{j'+1}\ge p\).
    Notice that \(j'\le j\). Since \(p>0=i_1\) and the support points are increasing, there is no support point in the interval \([p,i_{j'+1})\). Hence
    \[
    \Pr_{X\sim\D}[X\ge p]
    =
    \Pr_{X\sim\D}[X\ge i_{j'+1}].
    \]
    It follows that
    \begin{align*}
           \rev_{\D}(p) &= p \cdot \Pr_{X \sim \D}[X \geq p]\\
           &= p \cdot \Pr_{X \sim \D}[X \geq i_{j'+1}]\\
           &\leq
        i_{j'+1} \cdot \Pr_{X \sim \D}[X \geq i_{j'+1}] & (\text{Definition of $j'$})\\
        &\leq 2 - R(j') & (\text{Definition of $i_{j'+1}$})\\
        &\leq 2 - R(j)\,. & (\text{$R(\cdot)$ is decreasing}) 
    \end{align*}
Hence, for the revenue of the algorithm under the events 
$E_j, E'_j$ we have that
\[
    \E_{X_1,\ldots,X_j \sim \D^j}\left[\rev_{\D}(t_j(X_1,\ldots,X_j))|E_j, E'_{j}\right] \leq 2 -R(j) \,.
\]
To simplify the notation, let us
denote $X_{1:j} = X_1,\ldots,X_j.$
Thus, for the revenue of the algorithm we have that
\begin{align*}
    \E_{X_1,\ldots,X_j \sim {\D}^j}\left[\rev_{\D}(t_j(X_{1:j}))\right] &= 
    \E_{X_1,\ldots,X_j \sim {\D}^j}\left[\rev_{\D}(t_j(X_{1:j}))|E_j\right] \Pr[E_j] & \\
    &+ \E_{X_1,\ldots,X_j \sim {\D}^j}\left[\rev_{\D}(t_j(X_{1:j}))|\lnot E_j\right] \Pr[\lnot E_j] & (\text{Total probability)}\\
    &\leq \E_{X_1,\ldots,X_j \sim {\D}^j}\left[\rev_{\D}(t_j(X_{1:j}))|E_j\right] \Pr[E_j] & \\
    &+ 2\cdot \Pr[\lnot E_j] & (\opt = 2)\\
    &\leq \E_{X_1,\ldots,X_j \sim {\D}^j}\left[\rev_{\D}(t_j(X_{1:j}))|E_j, E'_{j}\right]\Pr[E'_j|E_j] \Pr[E_j] &\\
        &+\E_{X_1,\ldots,X_j \sim {\D}^j}\left[\rev_{\D}(t_j(X_{1:j}))\mid E_j, \lnot E'_{j}\right]\Pr[\lnot E'_j\mid E_j]\Pr[E_j] & \\
    &+2\cdot(1-\Pr[E_j]) & (\text{Total probability})\\
     &\leq\E_{X_1,\ldots,X_j \sim {\D}^j}\left[\rev_{\D}(t_j(X_{1:j}))|E_j, E'_{j}\right] \Pr[E'_j|E_j] \Pr[E_j] & \\
    &+ 2\Pr[\lnot E'_j|E_j]\Pr[E_j] + 2\cdot(1-\Pr[E_j])  & (\opt = 2)\\
    &\leq \E_{X_1,\ldots,X_j \sim {\D}^j}\left[\rev_{\D}(t_j(X_{1:j}))|E_j, E'_{j}\right]  \Pr[E_j] & \\
    & +2\Pr[\lnot E'_j|E_j] + 2\cdot(1-\Pr[E_j])  & ( \text{prob.} \leq 1) \\
    &\leq \E_{X_1,\ldots,X_j \sim {\D}^j}\left[\rev_{\D}(t_j(X_{1:j}))|E_j, E'_{j}\right]  \Pr[E_j] & \\
    & +\frac{R(j)}{2} + 2\cdot(1-\Pr[E_j])  & ( \text{Def. of $E_j'$}) \\
    & \leq \left(\E_{X_1,\ldots,X_j \sim {\D}^j}\left[\rev_{\D}(t_j(X_{1:j}))|E_j, E'_{j}\right]  - 2\right)\Pr[E_j] & \\
    &+ \frac{R(j)}{2} + 2 \,.& (\text{Rearrange})
\end{align*}

Now notice that
\[
     \left(\E_{X_1,\ldots,X_j \sim {\D}^j}\left[\rev_{\D}(t_j(X_1,\ldots,X_j))|E_j, E'_{j}\right]  - 2\right)\Pr[E_j]
\]
decreases as $\Pr[E_j]$ increases, because the expression in the
parentheses is negative. Since 
$\Pr[E_j] \geq 1-R(j)/2$ we have that
\begin{align*}
    \left(\E_{X_1,\ldots,X_j \sim {\D}^j}\left[\rev_{\D}(t_j(X_1,\ldots,X_j))|E_j, E'_{j}\right]  - 2\right)\Pr[E_j] \leq \\
    \left(\E_{X_1,\ldots,X_j \sim {\D}^j}\left[\rev_{\D}(t_j(X_1,\ldots,X_j))|E_j, E'_{j}\right]  - 2\right)(1-R(j)/2) \,.
\end{align*}
Moreover,
\[
    \E_{X_1,\ldots,X_j \sim {\D}^j}\left[\rev_{\D}(t_j(X_1,\ldots,X_j))|E_j, E'_{j}\right] \leq 2 -R(j) \,.
\]
Chaining these two inequalities gives us
\[
    \left(\E_{X_1,\ldots,X_j \sim {\D}^j}\left[\rev_{\D}(t_j(X_1,\ldots,X_j))|E_j, E'_{j}\right]  - 2\right)\Pr[E_j] \leq - R(j)(1-R(j)/2) = -R(j) +\frac{R^2(j)}{2} \,.
\]
Now chaining this inequality with the sequence of inequalities
we had above gives us
\begin{align*}
    \E_{X_1,\ldots,X_j \sim {\D}^j}\left[\rev_{\D}(t_j(X_1,\ldots,X_j))\right] &\leq -R(j) +\frac{R^2(j)}{2}  + \frac{R(j)}{2} + 2\\
    &= \opt - \frac{R(j)}{2} + \frac{R^2(j)}{2} \,.
\end{align*}
Since $R(j) < 1, \lim_{j \rightarrow \infty} R(j) = 0$, there is some
$j_0 \in \N$ such that for all $j \geq j_0$ it holds
that $\nicefrac{R^2(j)}{2} \leq \nicefrac{R(j)}{4}.$
Thus, for all $j \geq j_0$ we have
\begin{align*}
         \E_{X_1,\ldots,X_j \sim {\D}^j}\left[\rev_{\D}(t_j(X_1,\ldots,X_j))\right] &\leq \opt - \frac{R(j)}{2} + \frac{R(j)}{4} \\
         &\leq \opt -  \frac{R(j)}{4} \,.
\end{align*}
Recall that there is an infinite sequence of
$j \in \N$ such that $R(j) = \phi(j).$ 
Thus, the previous inequality shows that there is
an infinite sequence of $j \in \N$ such that
\[
   \E_{X_1,\ldots,X_j \sim {\D}^j}\left[\rev_{\D}(t_j(X_1,\ldots,X_j))\right] \leq
   \opt -  \frac{\phi(j)}{4} \,.
\]
This concludes the proof.   
\end{proof}

By inspecting the construction that witnesses
the arbitrarily slow rates lower bound in the previous result, one can observe that the distribution has the property that the
supremum of the revenue is not achieved by any price. It is 
worth highlighting that there are many natural distributions,
even regular
ones, that have this property. We refer the reader to \cref{section:Examples} 
for more detailed examples of such distributions.

\section{Near $1/\sqrt{n}$ Upper Bound for Unbounded Support}
\label{proof:upperUnbounded}

In this section, we prove the following upper bound that is achieved by the capped ERM algorithm.
The algorithm is simple and its pseudo-code is as follows.
\smallskip
\begin{mybox}
    \begin{center}
        The Algorithm of \Cref{thm:NearSqRootRates} (Capped Empirical Revenue Maximization)
    \end{center}
    \begin{enumerate}
        \item \textbf{Input}: Draw $S \sim \D^n$, non-decreasing function $g : \N \to \N$

        \item Construct the empirical distribution $\wh \D_n$

        \item Set $G = [0,g(n)]$ ~~~~~ \emph{\#Truncate the domain until point $g(n)$.}

        \item \textbf{return} the price $\wh t_n = \argmax_{t \in G} \rev_{\wh \D_n}(t)$.
    \end{enumerate}
\end{mybox}

\begin{lem}
[Near $1/\sqrt{n}$ Upper Bound]
Let $\mathbb{D}$ be the class of all valuation 
distributions over $\R_+$ 
whose optimal revenue is finite and 
can be attained by some price in the support.
Then, for any rate $R(n) \in \omega(n^{-1/2})$, 
there is an algorithm that learns $\mathbb{D}$ at universal rate $R(n).$
\end{lem}

\begin{proof}
    Let $\D \in \mathbb{D}$ be the target  distribution in the 
    family $\mathbb D$, and let $p^* \in \R_+$ be a price
    such that $\opt_\D = \rev_\D (p^*) = \sup_{p\geq 0} p \cdot \Pr_{v \sim \D}[v \geq p].$ By definition of the class $\mathbb D,$ such a $p^*$ is well-defined.

    Since $R(n) \in \omega(1/\sqrt{n})$ there exists some 
    non-decreasing function $g(n) \in \omega(1)$ and
    some ${n_1} \in \N$ 
    such that for all $n\geq {n_1}$ it holds that 
    $g(n)/\sqrt{n} \leq R(n).$ Let $n_2 \in \N$ be the smallest 
    number such that $g(n_2) \geq p^*.$ Since $g(n) \in \omega(1),$ such a number exists. Let also
    $n_3 = \max\{n_1, n_2\}.$

    We define the learning algorithm as follows:
for every \(n \in \N\) we create
the empirical distribution \(\hat \D_n\) on the samples
and output the price
\[
\wh t_n \in \argmax_{p \leq g(n)}
p \cdot \Pr_{v \sim \hat \D_n}[v \geq p].
\]
    Let $F_n$ be the empirical CDF 
    constructed by taking $n$ i.i.d. samples
    from $\D$ {and $F$ be the true CDF}.

    Then, notice that
     for all $n \geq n_3$ we also have
      $g(n) \geq p^*$. Recall the DKW inequality (\cref{lem:dkw-ineq}) which
      shows that for all $n \in \N,  \varepsilon > 0$
      it holds that
      \[
         \Pr[\sup_{x \in \R}\abs{F_n(x) - F(x)} > \varepsilon] \leq 2e^{-2n \varepsilon^2} \,.
      \]
      In order to translate the DKW bound to a bound on the expectation $\E[\sup_{x \in \R}\abs{F_n(x) - F(x)}]$,
      we can write
      \begin{align*}
          \E\left[\sup_{x \in \R}\abs{F_n(x) - F(x)}\right] &= \int_{0}^\infty \Pr[\sup_{x \in \R}\abs{F_n(x) - F(x)} > t] dt & (\text{Def. of expectation})\\
          &\leq \int_{0}^\infty 2 e^{-2n t^2} dt & (\text{DKW inequality})\\
          &= \sqrt{\frac{\pi}{2n}}\,. & 
      \end{align*}
      Let $\wh t_n$
      be the price that our algorithm outputs.
      Since $\wh t_n \leq g(n)$ it follows that
      \begin{equation}\label{eq:dkw-empirical-price}
        \E\left[\abs{\wh t_n \cdot (1-F_n(\wh t_n)) - \wh t_n \cdot (1-F(\wh t_n))} \right] \leq g(n) \cdot \sqrt{\frac{\pi}{2n}}  \,.
      \end{equation}
      Let us consider $n \geq n_3$. Similarly, we have
      \begin{equation}\label{eq:dkw-true-price}
        \E\left[\abs{ p^* \cdot (1-F_n(p^*)) -  p^* \cdot (1- F( p^*))}\right] \leq p^*\cdot \sqrt{\frac{\pi}{2n}} \,.
      \end{equation}
      By definition of the algorithm, it holds
      that
      \[
        \wh t_n \cdot (1- F_n(\wh t_n)) \geq p^* \cdot (1 - F_n(p^*)) \,.
      \]
      Chaining the above inequalities, we get
      that for all $n \geq n_3$
     \begin{align*}
         \E[\rev(\wh t_n)] &= \E\left[\wh t_n (1 - F(\wh t_n))\right] &\\
         &\geq \E\left[\wh t_n (1 - F_n(\wh t_n))\right] - g(n) \cdot \sqrt{\frac{\pi}{2n}} & (\text{\Cref{eq:dkw-empirical-price}})\\
         &\geq \E\left[ p^* (1 - F_n(p^*))\right] - g(n) \cdot \sqrt{\frac{\pi}{2n}} & (\text{Def. of alg.}) \\
             &\geq \E\left[ p^* (1 - F(p^*))\right] - (g(n) + p^*)\cdot \sqrt{\frac{\pi}{2n}}  & (\text{\Cref{eq:dkw-true-price}})\\
        &= \opt_\D - (g(n)+p^*)\sqrt{\frac{\pi}{2n}}.
    \end{align*}
    Since \(R(n)\in \omega(n^{-1/2})\) and \(g(n)/\sqrt n \le R(n)\) for all sufficiently large \(n\), we have
    \[
    (g(n)+p^*)\sqrt{\frac{\pi}{2n}}
    \le
    \sqrt{\frac{\pi}{2}}R(n)
    +
    p^*\sqrt{\frac{\pi}{2n}}
    \le
    C_\D R(n)
    \]
    for all sufficiently large \(n\), where the last inequality uses \(R(n)\sqrt n\to\infty\). Increasing \(C_\D\) if necessary to handle finitely many smaller values of \(n\), and taking \(c_\D=1\), we obtain
    \[
    \epsilon_n(\wh t_n,\D)
    =
    \opt_\D-\E[\rev_\D(\wh t_n)]
    \le
    C_\D R(c_\D n)
    \]
    for all \(n\).
\end{proof}

\section{Universal Lower Bound for Bounded Support }\label{proof:root-n-lb}
In this section we give the formal proof of the $o(1/\sqrt{n})$ lower bound.
\label{proof:UniversalLBbounded}
\begin{thm}
[$o(1/\sqrt{n})$-Rates (Lower Bound)]
\label{thm:LowerBoundBounded}
Consider the class of all valuation distributions $\mathbb D$ on $[0,1]$ with finite maximum revenue $\opt_D$. For any rate $R(n) \in o(n^{-1/2})$, the class $\mathbb D$ cannot be learned at a universal rate faster than $R(n).$
\end{thm}

\subsection{Lower Bound for Coin Estimation}

We introduce the following standard lemma which lower bounds the number of samples needed to distinguish between two similar coins. This lemma will be crucial for the universal lower bound and follows by observing that $\mathrm{KL}(\mathrm{Bernoulli}(p-\gamma) \parallel \mathrm{Bernoulli}(p+\gamma) ) \in O(\gamma^2/p)$ when $\gamma \ll p$ for small $p.$
\begin{lem}
[Lower Bound for Coin Estimation]
\label{lem:anti-bernstein}
Let $p$ be sufficiently small and $\gamma \ll p.$
Let $n = cp/\gamma^2$ for a universal absolute constant $c$.
Let $\wh \sigma_n : \{0,1\}^n \to \{-1,1\}$.
Let $\D_{-1} = \mathrm{Bernoulli}(p-\gamma)$ and $\D_{1} = \mathrm{Bernoulli}(p+\gamma)$.
Let $\sigma \sim \mathrm{Uniform}(\{-1,1\})$ and let $S$ be conditionally $\D_\sigma^n$ given $\sigma.$
Then 
\[
\Pr_S[\wh \sigma_n(S) \neq \sigma] \geq c'\,,
\]
for some universal constant $c'$.
\end{lem}

\subsection{Uniform Lower Bound for Bounded Support}
In this section, we provide a novel uniform lower bound for bounded support distributions, which we will use in the construction of the universal lower bound.

Before proving the exact statement, we provide its intuition.
In words, given a point $x \in (1/2,1]$ the next Lemma gives another point $x_{p,q}$ between $1/2$ and x such that there exists two collections of distributions $\sD_{-1}$ and $\sD_{1}$ and a number $\gamma > 0$ with the following property:
\begin{enumerate}
    \item Any price $t'$ that is on the left of $\frac{1}{2}(x_{p,q} + x)$ is $\gamma$ sub-optimal for any distribution in $\sD_{-1}$
    \item Any price $t'$ that is on the right of $\frac{1}{2}(x_{p,q} + x)$ is $\gamma$ sub-optimal for any distribution in $\sD_{1}$.
\end{enumerate}
This property will be key for our lower bound construction since, roughly speaking, we are going to force the learning algorithm to repeatedly guess whether the optimal price is on the left or right of $x_{p,q} + x$, and the above guarantees that we can force the algorithm make a wrong prediction.

\begin{figure}[!ht]
    \centering
    \includegraphics[width=0.7\linewidth]{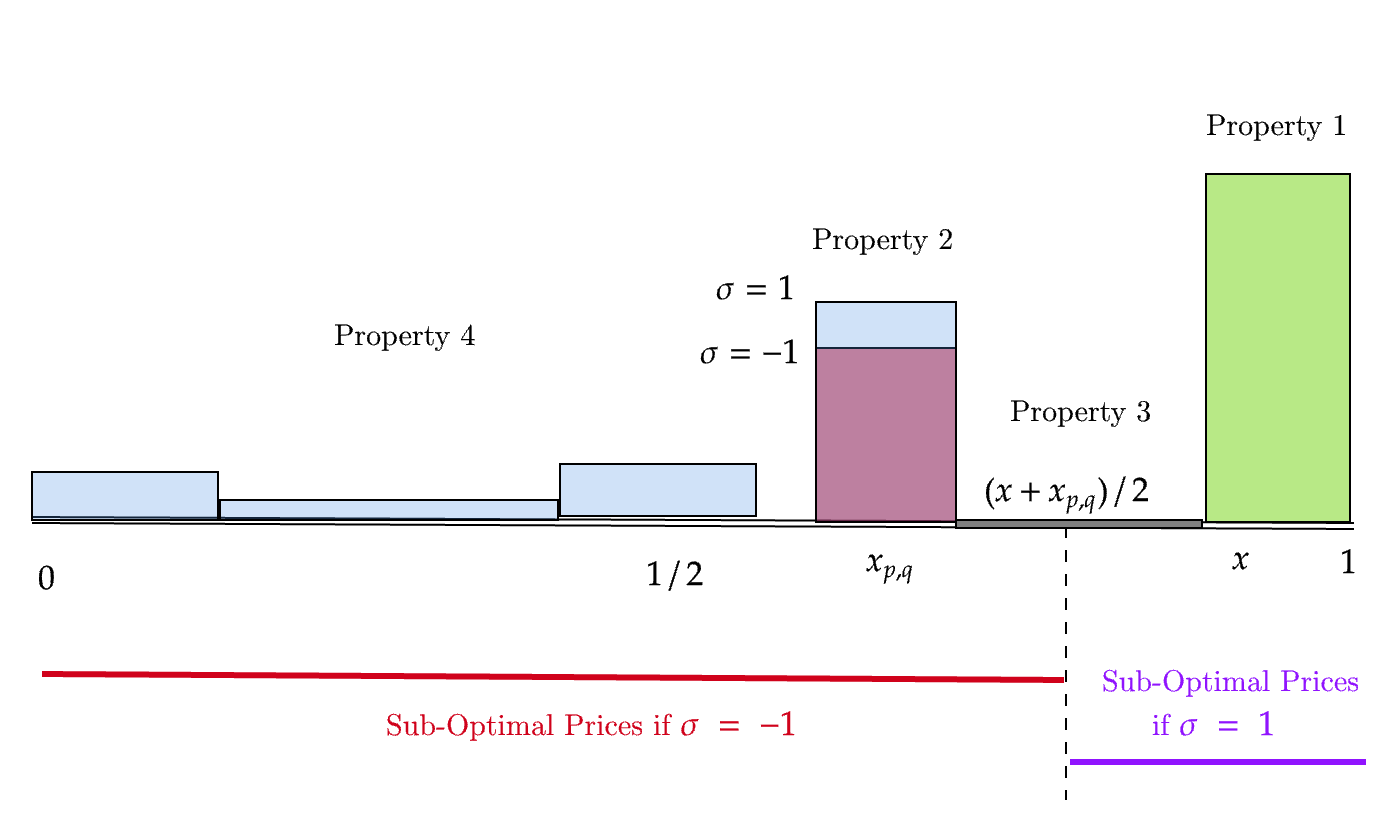}
    \caption{Illustration of distribution $\D$ satisfying the constraints of \Cref{lem:uniform-construction}. Properties 1-3 are specified directly by \Cref{lem:uniform-construction}. Property 4 places an additional restriction on how $\D$ should behave and we just illustrate it in a hand-wavy manner by putting some mass before the point $x_{p,q} - \gamma^2$. }
    \label{fig:uniform}
\end{figure}

\begin{lem}
[Uniform Lower Bound]
\label{lem:uniform-construction}
Fix any $x \in (1/2,1]$,
value $q \geq \frac{1}{2}$, 
and any sufficiently small $p > 0$ (for the given $x,q$).
Let $x_{p,q} = \frac{q}{p+q}x, x_{p,q} \in (1/2, x)$.
Let $\gamma > 0$ be any sufficiently small value 
($\gamma \ll p$, $\gamma \ll x-x_{p,q}$, $\gamma \ll x_{p,q}$).
For $\sigma \in \{-1,1\}$, 
let $\sD_{\sigma}$ be the set of all distribution $\D$ on $[0,1]$
such that 
\begin{enumerate}
\item $\D((x-\gamma^2,1]) \in [ q, q + \gamma^2]$, 
\item $\D([x_{p,q}-\gamma^2,x_{p,q}+\gamma^2)) \in [p + \sigma \gamma - \gamma^2, p + \sigma\gamma + \gamma^2]$, 
\item 
$\D([x_{p,q}+\gamma^2, x-\gamma^2]) = 0$,
\item The distribution $\D'$ with 
$\D'(\cdot| (0,1] \setminus (x_{p,q}-\gamma^2,x_{p,q}+\gamma^2)) = \D(\cdot | (0,1] \setminus (x_{p,q}-\gamma^2,x_{p,q}+\gamma^2))$, $\D'((0,1] \setminus (x_{p,q}-\gamma^2,x_{p,q}+\gamma^2)) = \D( (0,1] \setminus (x_{p,q}-\gamma^2,x_{p,q}+\gamma^2))$, 
and $\D'(0) = 1 - \D((0,1] \setminus (x_{p,q}-\gamma^2,x_{p,q}+\gamma^2))$
has 
$\argmax_t t \D'([t,1]) \in [x-\gamma^2,x]$
and 
$\argmax_{t \leq x_{p,q} - \gamma^2} t \D'( [t,1] ) = x_{p,q} - \gamma^2$. 
\end{enumerate}
Then $x_{p,q}$ satisfies:
\begin{itemize}
\item $\forall \D \in \sD_{-1}$, 
any $t' \leq \frac{1}{2}(x_{p,q}+x)$ has $t' \D( [t',1]) < \max_{t} t \D([t,1]) - \frac{\gamma}{4}$.
\item $\forall \D \in \sD_{1}$, 
any $t' \geq \frac{1}{2}(x_{p,q}+x)$ has 
$t' \D( [t',1] ) < \max_{t} t \D([t,1]) - \frac{\gamma}{4}$.
\end{itemize}
\end{lem}

{The family $\sD_\sigma$ places constraints $1-4$ on the valid distributions. An illustration of a member of $\sD_\sigma$  can be found in Figure \ref{fig:uniform}. Constrains 1--3 are simply mass constraints on $\D$ where only Requirement 2 depends on the sign of $\sigma$ (changing the mass around $x_{p,q}$ as seen in Figure \ref{fig:uniform}). Let us now sketch Requirement 4. This condition, roughly speaking, constructs another distribution $\D'$ by transporting the mass of $\D$ from the interval $(x_{p,q} - \gamma^2, x_{p,q} + \gamma^2)$ to 0. The restriction is that $\D'$ should have maximum revenue near $x$ and if one looks only at prices before $x_{p,q} - \gamma^2$, then the optimal price is exactly at $x_{p,q} - \gamma^2$. 

Under the above constraints, the point $x_{p,q}$ given by the above Lemma satisfies the conclusion. We proceed with the proof.}

\begin{proof}
The lower bound construction gets as input a point $x \in (1/2,1]$ and some value $q$ and returns some small value $p$, a new point $x_{p,q}$ (which will be on the left of $x$ but very close to it) and some very small slackness parameter $\gamma.$ These three parameters will be chosen so that the conclusion of the Lemma holds.

\paragraph{Choosing $x_{p,q}$} Let us first determine the new point: $x_{p,q} = \frac{q}{p+q}x < x$, 
which, since $x > \frac{1}{2}$, 
for sufficiently small $p$ satisfies $x_{p,q} > \frac{1}{2}$
(and moreover satisfies the property $x_{p,q} \to x$ as $p \to 0$).

It remains to show that the choice of $x_{p,q}$ satisfies:
\begin{itemize}
\item (Condition 1) $\forall \D \in \sD_{-1}$, 
any $t' \leq \frac{1}{2}(x_{p,q}+x)$ has $t' \D( [t',1]) < \max_{t} t \D([t,1]) - \frac{\gamma}{4}$.
\item (Condition 2) $\forall \D \in \sD_{1}$, 
any $t' \geq \frac{1}{2}(x_{p,q}+x)$ has 
$t' \D( [t',1] ) < \max_{t} t \D([t,1]) - \frac{\gamma}{4}$.
\end{itemize}
\paragraph{Checking Condition 1} 
Let $\D \in \sD_{-1}$.
Let $t' \leq \frac{1}{2}(x_{p,q}+x)$.
If $t' \geq x_{p,q} + \gamma^2$, 
then 
\begin{align*}
t' \D([t',1]) 
& = t' \D([x-\gamma^2,1])
\leq \frac{1}{2} ( x_{p,q} + x ) \D([x-\gamma^2,1]).
\end{align*}

Since $\gamma \ll x-x_{p,q}$ implies $x_{p,q} \leq x - 4\gamma - 2\gamma^2$, this is at most
\begin{align*}
\frac{1}{2} ( x - 4\gamma - 2\gamma^2 + x) \D([x-\gamma^2,1])
= ( x - \gamma^2 ) \D([x-\gamma^2,1]) - 2\gamma \D([x-\gamma^2,1]).
\end{align*}
Since $\D([x-\gamma^2,1]) \geq q \geq \frac{1}{2}$, 
this is at most
\begin{align*}
( x - \gamma^2 ) \D([x-\gamma^2,1]) - \gamma
\leq \max_t t \D([t,1]) - \gamma.
\end{align*}

On the other hand, 
consider any $t' \in [x_{p,q} - \gamma^2, x_{p,q} + \gamma^2)$.
We have 
\begin{align*} 
t' \D([t',1]) 
\leq (x_{p,q}+\gamma^2) \D([ x_{p,q}-\gamma^2,1])
\leq (x_{p,q}+\gamma^2) \left( p + q + 2\gamma^2 - \gamma \right).
\end{align*}
By definition of $x_{p,q}$, we have 
$p+q = \frac{x}{x_{p,q}} q$, so that the above equals
\begin{align*}
& (x_{p,q}+\gamma^2) \left( \frac{x}{x_{p,q}}q + 2\gamma^2 - \gamma \right)
\\ & = x q - \gamma x_{p,q}
+ (x_{p,q}+\gamma^2) \left( 2\gamma^2 \right) + \gamma^2 \frac{x}{x_{p,q}}q - \gamma^3
\\ & < x q - \frac{\gamma}{2} + 4\gamma^2.
\end{align*}
where we have used that $x_{p,q}+\gamma^2 \leq x \leq 1$, 
$x_{p,q} > \frac{1}{2}$.
Additionally, 
\begin{equation*}
x q 
\leq (x-\gamma^2) q + \gamma^2
\leq (x-\gamma^2) \D([x-\gamma^2,1]) + \gamma^2
\leq \max_t t \D([t,1]) + \gamma^2,
\end{equation*}
so that altogether we have
\begin{equation*}
t' \D([t',1]) < \max_t t \D([t,1]) - \frac{\gamma}{2} + 5 \gamma^2
< \max_t t \D([t,1]) - \frac{\gamma}{4},
\end{equation*}
for $\gamma$ sufficiently small to satisfy $5 \gamma^2 < \frac{\gamma}{4}$.

Finally, 
consider any $t' < x_{p,q} - \gamma^2$.
If $t' \leq 0$, we clearly have $t' \D([t',1]) \leq 0 < \max_t t \D([t,1]) - \gamma$ since $\max_t t \D([t,1]) \geq (x - \gamma^2) q > 2\gamma$.
So suppose $t' \in (0,x_{p,q}-\gamma^2)$.
We then have
\begin{align*}
& t' \D([t',1]) 
= t' \D'([t',1]) + t' \D((x_{p,q}-\gamma^2,x_{p,q}+\gamma^2))
\\ & \leq (x_{p,q}-\gamma^2) \D'([ x_{p,q}-\gamma^2,1]) + t' \D((x_{p,q}-\gamma^2,x_{p,q}+\gamma^2)) 
\leq (x_{p,q}-\gamma^2) \D([ x_{p,q}-\gamma^2,1]). 
\end{align*}
where the first inequality is by the condition on $\D'$
and the second inequality is by definition of $\D'$
and the assumption $t' < x_{p,q}-\gamma^2$.
Since the last expression above is at most 
$\max_{x_{p,q} - \gamma^2 \leq t < x_{p,q} + \gamma^2} t \D([t,1])$, 
our analysis of the case $t' \in [x_{p,q} - \gamma^2, x_{p,q} + \gamma^2)$ above implies this is less than 
$\max_t t \D([t,1]) - \frac{\gamma}{4}$.

\paragraph{Checking Condition 2}
Now we turn to the case $\D \in \sD_1$,
and consider any $t' \geq \frac{1}{2}(x_{p,q}+x)$.
Letting $t'' = \argmax_t t \D'([t,1])$, 
we have
\begin{align*}
t' \D([t',1])
= t' \D'([t',1])
\leq t'' \D'([t'',1])
\leq x \D([x-\gamma^2,1])
\end{align*}
where the last inequality is by our condition on $\D'$.
On the other hand, 
\begin{align*}
\max_t t \D([t,1])
\geq (x_{p,q}-\gamma^2) \D([x_{p,q}-\gamma^2,1])
\geq (x_{p,q} - \gamma^2) \left( p + q + \gamma - \gamma^2 \right).
\end{align*}
Again, we have 
$p+q = \frac{x}{x_{p,q}} q$,
so that the last expression above equals
\begin{align*}
(x_{p,q} - \gamma^2) \left( \frac{x}{x_{p,q}} q + \gamma - \gamma^2 \right)
\geq x q + x_{p,q} \gamma - 4 \gamma^2
> x q + \frac{\gamma}{2} - 4 \gamma^2,
\end{align*}
where the first inequality follows from $x_{p,q} > \frac{1}{2}$, $x q \leq 1$, and $\gamma < 1$,
and the second inequality follows from $x_{p,q} > \frac{1}{2}$.
Thus, 
\begin{equation*}
x q < \max_t t \D([t,1]) - \frac{\gamma}{2} + 4 \gamma^2.
\end{equation*}
Moreover,
\begin{equation*}
x q 
\geq x (q + \gamma^2) - \gamma^2 
\geq x \D([x-\gamma^2,1]) - \gamma^2.
\end{equation*}
Altogether, we have that
\begin{align*}
t' \D([t',1])
\leq x \D([x-\gamma^2,1])
\leq x q + \gamma^2
< \max_t t \D([t,1]) - \frac{\gamma}{2} + 5 \gamma^2
< \max_t t \D([t,1]) - \frac{\gamma}{4},
\end{align*}
supposing $\gamma$ is sufficiently small that $5 \gamma^2 < \frac{\gamma}{4}$.
\end{proof}

\subsection{Universal Lower Bound for Bounded Support}
\label{universalLBbounded}
The next result shows that for the revenue maximization problem and for unrestricted distributions $P$ on $[0,1]$, there is a
near-square-root lower bound:

\begin{lem}
[Near $1/\sqrt{n}$ Lower Bound]
For any algorithm $\wh t_n : [0,1]^n \to [0,1],$ for any rate $R(n) \in o(n^{-1/2})$, there exists a distribution $\D$ on $[0,1]$ such that, for infinitely many $n \in \N$,
\[
\E_{S \sim \D^n}
[\max_{t} t \D([t,1]) - \wh t_n(S) \D([\wh t_n(S),1]) )]
\geq R(n)\,.
\]
\end{lem}

\begin{proof}
Fix any rate $R(n) \in o(n^{-1/2}).$ We will first define an ensemble of distributions, for which we will argue that, for any revenue maximization algorithm $\wh t_n$, at least one of the distributions in the family has the property claimed in the lemma, i.e., there exists some distribution $\D$ in the ensemble such that,
for infinitely many $n \in \N$,
\[
\E_{S \sim \D^n}
[\max_{t} t \D([t,1]) - \wh t_n(S) \D([\wh t_n(S),1]) )]
\geq R(n)\,.
\]

\paragraph{Construction of the Ensemble} 
Let $\sigma_{-1} = \sigma_0 = 1$, 
and $\sigma_1,\sigma_2,\ldots \in \{-1,1\}$. 
For $\sigma = (\sigma_0,\sigma_1,\sigma_2,\ldots)$, 
we define a distribution $\D_{\sigma}$ as follows.
We inductively define a sequence of values $x \in (1/2,1]$
serving as the support of the distribution $\D_{\sigma}$.

As a first step, let $x_{()} = 1$ and $\D_{\sigma}(x_{()}) = p_{()} := \frac{1}{2}$. 
Let also $\gamma_{()} = 0$. It remains to place the remaining $1/2$-mass to other points. We will iteratively use the construction of the uniform lower bound of \Cref{lem:uniform-construction} to do that.

We are thinking of the construction of the uniform lower bound as a mapping $L$ that maps pairs $(x,q)$ to specific tuples $(p, x_{p,q}, \gamma)$ that allow for the uniform lower bound construction to go through.

Let us call $L(x := x_{()}, q := 1/2)$ to get $(p_{(\sigma_0)}, x_{p(\sigma_0), 1/2}, \gamma_{(\sigma_0)}).$ Note that
 $p_{(\sigma_0)} > 0$
 and $\gamma_{(\sigma_0)} \in (0,1/4)$ are sufficiently small values in order to satisfy the 
requirements of Lemma~\ref{lem:uniform-construction}.

Let the second point of the support be $x_{(\sigma_0)} = x_{(1)}$ which has the value $x_{p_{(\sigma_0)},1/2}$.
Let $n_{(\sigma_0)} = 1$. The mass of the point $x_{(\sigma_0)}$ is determined in the next iteration.

For the general inductive argument, we will use the notation
$[k'] = \{0,1,\ldots,k'\}$ and
$\sigma_{[k']} = (\sigma_0,\ldots,\sigma_{k'})$ with
$\sigma_{[-1]} = ()$.

For $k = 1:$ Up to now, we have defined the points $x_{()}, x_{(\sigma_0)}$, the point $x_{()}$ has mass 1/2 (and the mass of the point $x_{(\sigma_0)}$ is not defined yet), the defined gaps are $\gamma_{()}, \gamma_{(\sigma_0)}$ and the other parameters are $p_{()}, p_{(\sigma_0)}, n_{(\sigma_0)}$. We will now have to define the mass $\D_\sigma(x_{(\sigma_0)})$ and the new elements $x_{(\sigma_0, \sigma_1)}, p_{(\sigma_0, \sigma_1)}, \gamma_{(\sigma_0, \sigma_1)}$ and $n_{(\sigma_0, \sigma_1)}$ based on $\sigma_1.$

More generally, suppose $k \in \nats$ is such that 
we have defined distinct values 
$x_{\sigma_{[k']}} \in (1/2,1]$ for all $k' \in \{0,\ldots,k-1\}$,
along with $n_{\sigma_{[k']}} \in \nats$, 
$p_{\sigma_{[k']}} > 0$,
and $\gamma_{\sigma_{[k']}} > 0$
and corresponding probability values 
$\D_{\sigma}(x_{\sigma_{[k'-1]}}) \in [p_{\sigma_{[k'-1]}}-\gamma_{\sigma_{[k'-1]}},p_{\sigma_{[k'-1]}}+\gamma_{\sigma_{[k'-1]}}]$,
and satisfying the constraint $\sum_{k' \in \{-1,\ldots,k-1\}} (p_{\sigma_{[k']}}+\gamma_{\sigma_{[k']}}) < 1$, 
all based purely on $\sigma_0,\ldots,\sigma_{k-1}$. We now have to define
$\D_{\sigma}(x_{\sigma_{[k-1]}})$, 
$x_{\sigma_{[k]}}$, $p_{\sigma_{[k]}}$, $\gamma_{\sigma_{[k]}}$, 
and $n_{\sigma_{[k]}}$,
based on $\sigma_k$.

\paragraph{Mass of $x_{(\sigma_0, ..., \sigma_{k-1})}$}
We will use the previous values $p_{\sigma_{[k-1]}}$ and $\gamma_{\sigma_{[k-1]}}$ returned by the uniform lower bound mapping together with the sign $\sigma_k$:
Set
\begin{equation*} 
\D_{\sigma}(x_{\sigma_{[k-1]}}) = p_{\sigma_{[k-1]}} + \sigma_k \gamma_{\sigma_{[k-1]}}.
\end{equation*}
Hence,
depending on the sign of $\sigma_k$, the mass on that point (the parent node) is either slightly less than $p_{\sigma_{[k-1]}}$ or slightly higher.

\paragraph{Calling Uniform Lower Bound to get $x_{\sigma_{[k]}}$, $p_{\sigma_{[k]}}$ and $\gamma_{\sigma_{[k]}}$}

Now, we have to decide the input parameters that we will use to call $L(x,q)$.  We set:
\begin{enumerate}
    \item $x = x^*_{\sigma_{[k]}} := \min\{ x_{\sigma_{[k']}} : 0 \leq k' \leq k-1, \sigma_{k'} = 1 \}$, i.e., we look at all points in the prefix that end with 1 (at least one exists since $\sigma_0 = 1$) and pick the leftmost among them.

    \item $q = q_{\sigma_{[k]}} := \sum_{-1 \leq k' \leq k-1 : \sigma_{k'} = 1} \D_{\sigma}(x_{\sigma_{[k']}}).$  
\end{enumerate}

Let $p, x_{p,q}$ be the points returned by $L(x,q)$ (we will tune the sufficiently small $\gamma$ right after).

\begin{enumerate}
    \item We set  $p_{\sigma_{[k]}} = p$ for $p > 0$
sufficiently small for Lemma~\ref{lem:uniform-construction} to hold 
for this $x,q$ 
and also sufficiently small so that
$p + \sum_{k' \in \{-1,\ldots,k-1\}} (p_{\sigma_{[k']}}+\gamma_{\sigma_{[k']}}) < 1$.

\item We also set $x_{\sigma_{[k]}} = x_{p,q}.$
\end{enumerate}
For $\bar{\gamma} = \bar{\gamma}_{\sigma_{[k]}} := \min_{0 \leq k' \leq k-1} \gamma_{\sigma_{[k']}}$,
we have $x_{p,q} > x - \bar{\gamma}^2$, 
$p < 2^{-k-1} \bar{\gamma}^2$
and $p < \frac{c'}{2^{k+2} n_{\sigma_{[k-1]}}}$ 
(where $c'$ is from Lemma~\ref{lem:anti-bernstein}).

We now tune $\gamma_{\sigma_{[k]}}$, which will be chosen as a function of $n_{\sigma_{[k]}}$. Let $n_{\sigma_{[k]}} \in \nats$ 
with $n_{\sigma_{[k]}} > n_{\sigma_{[k-1]}}$
be sufficiently large so that, 
for $c,c'$ as in Lemma~\ref{lem:anti-bernstein}, 
$R(n_{\sigma_{[k]}}) \leq \frac{c'}{16}\sqrt{\frac{c p_{\sigma_{[k]}}}{n_{\sigma_{[k]}}}}$
(recalling $R(n) \in o(n^{-1/2})$)
and sufficiently large so that 
$\gamma_{\sigma_{[k]}} := \gamma = \sqrt{\frac{c p_{\sigma_{[k]}}}{n_{\sigma_{[k]}}}}$
satisfies the requirements of Lemma~\ref{lem:uniform-construction}
(for the above $x,q,p,x_{p,q}$)
along with the requirements
$\gamma \ll \bar{\gamma}$ 
and 
$\sum_{k' \in \{-1,\ldots,k\}} (p_{\sigma_{[k']}}+\gamma_{\sigma_{[k']}}) < 1$.

\paragraph{Remaining Mass}
We now repeat the above construction inductively. After the above construction defines these $\D_{\sigma}(x_{\sigma_{[k]}})$ values for all $k \in \{-1,0,1,\ldots\}$, 
we complete the definition of $\D_{\sigma}$ by setting 
$\D_{\sigma}(0) = 1 - \sum_{k \geq -1} \D_{\sigma}(x_{\sigma_{[k]}})$.

\paragraph{Basic Property} The idea of the above construction is that $\D_\sigma$ satisfies the property that, for each $k$, for the values 
\[
(x,q,p,x_{p,q},\gamma) = (x^*_{\sigma_{[k]}}, q_{\sigma_{[k]}}, p_{\sigma_{[k]}}, x_{\sigma_{[k]}}, \gamma_{\sigma_{[k]}})
\]
defined in the above construction, $\D_\sigma$ belongs to the set $\sD_{\sigma_k}$ of the uniform lower bound of \Cref{lem:uniform-construction}. 

\begin{claim}
For given $k$ and for given $\sigma_1,...,\sigma_k$, the property $\D_\sigma \in \sD_{\sigma_k}$
holds regardless of the values $\sigma_{k'}$ for $k' > k.$    
\end{claim}
 This is due to the allowances of $\pm \gamma^2$ appearing in the conditions defining $\sD_{\sigma_k}$ (and the choices of $p$ above, 
to ensure subsequent points are sufficiently close to their associated $x$ to be within these $\pm \gamma^2$,
and the effect of these points on the probabilities of these 
sets are also within the $\pm \gamma^2$ allowance there).

\paragraph{Reduction to Binary Prediction} 
Thus, for each $k$, for $n = n_{\sigma_{[k]}}$, 
the learner must \emph{choose} whether its $\hat{t}_n$
is closer to $x_{\sigma_{[k]}}$ 
or closer to $x^*_{\sigma_{[k]}}.$
The latter is correct if $\sigma_k = -1$, 
and otherwise the former is correct with $\sigma_k = 1$  
(and in this case 
$x^*_{\sigma_{[k+1]}}$ will be $x_{\sigma_{[k]}}$).
Since $x_{\sigma_{[k]}}$ is set to the value $x_{p,q}$ 
from Lemma~\ref{lem:uniform-construction}
and $x^*_{\sigma_{[k]}}$ serves as the value $x$ in that context, 
we know that if the learner guesses incorrectly, 
its loss compared to the optimal $t$ is at least $\frac{\gamma_{\sigma_{[k]}}}{4}$.

Our main goal is to argue that there is a choice of $\sigma$'s such that 
the given learner $\wh t_n$ will indeed guess incorrectly with significant probability. In order to establish the existence of such a distribution, we will start from a random one and then use a ``probabilistic method''-type of argument.

\paragraph{Picking a Random Distribution}
Let $\bsigma_1,\bsigma_2,\ldots$ be i.i.d.\ $\mathrm{Uniform}(\{-1,1\})$ random variables.
Let $\bsigma = (\sigma_0,\bsigma_1,\bsigma_2,\ldots)$.
Let $X_1,X_2,\ldots$ be conditionally i.i.d. drawn from  $P_{\bsigma}$ given $\bsigma$.
For each $k \in \nats$, 
let $\hat{t}_{(k)} = \hat{t}_{n_{\bsigma_{[k]}}}(\{X_1,\ldots,X_{n_{\bsigma_{[k]}}}\})$,
and let 
\begin{equation*} 
\hat{\sigma}_k = \begin{cases}
    1, & \text{ if } |\hat{t}_{(k)} - x_{\bsigma_{[k]}}| \leq |\hat{t}_{(k)} - x^*_{\bsigma_{[k]}}|
    \\ -1 & \text{ if } |\hat{t}_{(k)} - x_{\bsigma_{[k]}}| > |\hat{t}_{(k)} - x^*_{\bsigma_{[k]}}|
\end{cases}\,.
\end{equation*}
Let us quantify what it means to make a mistake in the prediction.
By the above, we have that if
$\hat{\sigma}_{k+1} \neq \bsigma_{k+1}$, then the expected revenue is less than the expected maximum one and their gap is:
\begin{align}
\hat{t}_{(k)} \D_{\bsigma}([\hat{t}_{(k)},1]) 
& < \max_t t \D_{\bsigma}([t,1]) - \frac{1}{4}\gamma_{\bsigma_{[k]}} \notag \\ & \leq \max_t t \D_{\bsigma}([t,1]) - \frac{1}{4}\sqrt{\frac{c p_{\sigma_{[k]}}}{n_{\sigma_{[k]}}}}
\leq \max_t t \D_{\bsigma}([t,1]) - \frac{4}{c'}R (n_{\bsigma_{[k]}}).
\label{eqn:phi-loss-gap}
\end{align}

\paragraph{Probability of Wrong Prediction}

Fix any $k \in \nats$.
Let $E_k$ denote the event that
every $t \leq n_{\bsigma_{[k]}}$ has 
$X_t \notin \{x_{\bsigma_{[k']}} : k' > k \}$.
Note that, by our choices of $p_{\sigma_{[k']}}$ above, 
$E_k$ has conditional probability at least $1 - \frac{c'}{2}$
given $\bsigma_{[k]}$.
Moreover, given this event, 
the only information about $\bsigma_{k+1}$ 
in these samples $X_t$, $t \leq n_{\bsigma_{[k]}}$, is in the indicators $\ind[ X_t = x_{\bsigma_{[k]}} ]$
and $\ind[ X_t = 0 ]$.
\footnote{ {It is important to note that  $\ind[ X_t = 0 ]$ also gives information about $\bsigma_{k+1}$  However, since estimating $\D_{\sigma}(0)$ up to $\pm \gamma_{\bsigma_{[k+1]}}$ based on the samples $\ind[X_t = 0]$ is even harder than estimating $\D_{\sigma}(x_{\bsigma_{[k]}})$ up to $\pm \gamma_{\bsigma_{[k+1]}}$ based on the samples $\ind[X_t = x_{\bsigma_{[k]}}]$. The reason of that is the following: Bernstein's inequality 
bounds the number of samples required for estimating a coin of bias $p$ with $\gamma$ accuracy as at least $p/\gamma^2$; hence, smaller $p$'s make the problem easier. However, $\D_{\sigma}(0) \in \Omega(1)$ and so those samples do not even get the factor of $p$ in their sample complexity from Bernstein's inequality.}  
}
The variable $\bsigma_{k+1}$ is conditionally 
$\mathrm{Uniform}(\{-1,1\})$ given $\bsigma_{[k]}$ (as they are independent),
and the indicators $\ind[ X_t = x_{\bsigma_{[k]}} ]$
are conditionally $\mathrm{Ber}( p_{\bsigma_{[k]}} + \bsigma_{k+1} \gamma_{\bsigma_{[k]}} )$ 
given $\bsigma_{[k]}$ and $\bsigma_{k+1}$.

Since our choice of $\gamma_{\bsigma_{[k]}}$ 
ensures $n_{\bsigma_{[k]}} \leq \frac{c p_{\bsigma_{[k]}}}{\gamma_{\bsigma_{[k]}}^2}$, 
\Cref{lem:anti-bernstein} implies 
\begin{align*}
\Pr\!\left( \hat{\sigma}_{k+1} \neq \bsigma_{k+1} \middle| \bsigma_{[k]}, E_k \right) \geq c'.
\end{align*}
Together we have 
\begin{align*}
\Pr\!\left( \hat{\sigma}_{k+1} \neq \bsigma_{k+1} \middle| \bsigma_{[k]} \right)
& \geq \Pr\!\left( \left\{ \hat{\sigma}_{k+1} \neq \bsigma_{k+1} \right\} \cap E_k \middle| \bsigma_{[k]} \right)
\\ & = \Pr\!\left( \left\{ \hat{\sigma}_{k+1} \neq \bsigma_{k+1} \right\} \middle| \bsigma_{[k]}, E_k \right) \Pr(E_k | \bsigma_{[k]} )
\geq c' \left( 1 - \frac{c'}{2} \right) \geq \frac{c'}{2}.
\end{align*}
Given the above lower bound to the probability of wrong prediction, we can get by \eqref{eqn:phi-loss-gap}, 
\begin{equation}
\label{eqn:prob-lb}
\Pr\!\left( \max_t t \D_{\bsigma}([t,1]) - \hat{t}_{(k)} \D_{\bsigma}([\hat{t}_{(k)},1]) > \frac{4}{c'} R(n_{\bsigma_{[k]}}) \middle| \bsigma_{[k]} \right) \geq \frac{c'}{2}.
\end{equation}

\paragraph{Proof of Universal Lower Bound} 

We now proceed to argue the $R(n)$ lower bound
for infinitely many $n$.  
To show this, note that it suffices to show the following claim.
\begin{claim}
    It holds that
    \begin{align*}
& \sup_{\D ~\mathrm{ on }~ [0,1]} \limsup_{n \to \infty} \frac{1}{R(n)} \E_{S \sim \D^n}\left[ \max_t t \D([t,1]) - \hat{t}_n(S) \D([\hat{t}_n(S),1]) \right]
> 1.
\end{align*}

\end{claim}
We will first lower bound the $\sup$ over all $\D$ by the average over $\D_\sigma$'s. This means that the left-hand-side is at least
\begin{align}
& \E_{\bsigma} \!\left[ \limsup_{n \to \infty} \frac{1}{R(n)} \E_{X_1,...,X_n \sim \D_{\bsigma}^n}
\!\left[ 
\max_t t \D_{\bsigma}([t,1]) - \hat{t}_n(X_1,\ldots,X_n) \D_{\bsigma}([\hat{t}_n(X_1,\ldots,X_n),1]) \middle| \bsigma \right] \right] 
\notag \\ & \geq \E_{\bsigma}\!\left[ \limsup_{k \to \infty} \frac{1}{R(n_{\bsigma_{[k]}})} \E_{S \sim \D_{\bsigma}^{n_{\bsigma_{[k]}}}}\!\left[ \max_t t \D_{\bsigma}([t,1]) - \hat{t}_{(k)} \D_{\bsigma}([\hat{t}_{(k)},1]) \middle| \bsigma \right] \right]. \label{eqn:pre-markov}
\end{align}
For the nonnegative random variable inside the expectation, we use the elementary bound
\[
\E[Z]\ge a\Pr(Z\ge a).
\]
Thus,
\begin{align*} 
& \E_{\D_{\bsigma}^{n_{\bsigma_{[k]}}}}
\!\left[ \max_t t \D_{\bsigma}([t,1]) - \hat{t}_{(k)} \D_{\bsigma}([\hat{t}_{(k)},1]) \middle| \bsigma \right]
\\ & \geq \frac{4}{c'} R(n_{\bsigma_{[k]}}) \cdot \Pr_{\D_{\bsigma}^{n_{\bsigma_{[k]}}}}\!\left( \max_t t \D_{\bsigma}([t,1]) - \hat{t}_{(k)} \D_{\bsigma}([\hat{t}_{(k)},1]) \geq \frac{4}{c'} R(n_{\bsigma_{[k]}}) \middle| \bsigma \right),
\end{align*}
so that \eqref{eqn:pre-markov} is no smaller than 
\begin{align*}
& \E_{\bsigma}\!\left[ \limsup_{k \to \infty} \frac{4}{c'} \cdot \Pr_{\D_{\bsigma}^{n_{\bsigma_{[k]}}}}\!\left( \max_t t \D_{\bsigma}([t,1]) - \hat{t}_{(k)} \D_{\bsigma}([\hat{t}_{(k)},1]) \geq \frac{4}{c'} R(n_{\bsigma_{[k]}}) \middle| \bsigma \right) \right].
\end{align*}
Since 
$\frac{4}{c'} \Pr\!\left( \max_t t \D_{\bsigma}([t,1]) - \hat{t}_{(k)} \D_{\bsigma}([\hat{t}_{(k)},1]) \geq \frac{4}{c'} R(n_{\bsigma_{[k]}}) \middle| \bsigma \right) \leq \frac{4}{c'}$ (i.e., it is bounded), 
reverse Fatou's lemma (and linearity of the expectation)
implies the expectation above is at least
\begin{align*}
& \limsup_{k \to \infty} \frac{4}{c'} \cdot \E_{\bsigma}\!\left[ \Pr\!\left( \max_t t \D_{\bsigma}([t,1]) - \hat{t}_{(k)} \D_{\bsigma}([\hat{t}_{(k)},1]) \geq \frac{4}{c'} R(n_{\bsigma_{[k]}}) \middle| \bsigma \right) \right].
\end{align*}
By the law of total probability, this equals
\begin{align*}
\limsup_{k \to \infty} \frac{4}{c'} \cdot \E_{\bsigma}\!\left[ \Pr\!\left( \max_t t \D_{\bsigma}([t,1]) - \hat{t}_{(k)} \D_{\bsigma}([\hat{t}_{(k)},1]) \geq \frac{4}{c'} R(n_{\bsigma_{[k]}}) \middle| \bsigma_{[k]} \right) \right] \geq 2,
\end{align*}
where the last inequality is due to \eqref{eqn:prob-lb}. This shows that there exists a distribution in the ensemble such that the lower bound holds infinitely often.

\end{proof}

\section{Universal Upper Bound for Bounded Support}
\label{proof:upperBounded}

The main result of this section is an algorithm for bounded support distributions that gets $o(n^{-1/2})$ universal rate. Without loss of generality, through appropriate rescaling and shifting, it suffices to state the result for $[0,1].$

\begin{thm}
[$o(1/\sqrt{n})$-Rates (Upper Bound)]
\label{thm:FastRatesBounded}
Consider the class of all valuation distributions $\mathbb D$ on $[0,1]$. Then $\mathbb D$ is universally learnable at rate $o(n^{-1/2}).$
\end{thm}

In particular, our algorithm will be ERM: for any distribution $\D$ on $[0,1]$, let $S \sim \D^n$ and let $\wh t_n$ be any element of the set $ \argmax_t \rev_{\wh D_n}(t)$, where $\wh D_n$ is the empirical estimate of $\D$ based on $S.$ Then
\[
\opt_\D - \E_S \rev_\D(\wh t_n) \in o(n^{-1/2}).
\]

\paragraph{Notation.}  We will first need the notation $g_\D(t,t') = \min \{t,t'\} \D([\min\{t,t'\}, \max\{t,t'\}))$. Crucially, the upper end of the interval is open; this is to avoid pathological cases where there are point masses
at the upper end of the interval.
For any $\epsilon \geq 0$, define the set of prices that are $\epsilon$-close to the optimal revenue under $\D$:
\[
T(\epsilon) = \{t : \rev_\D(t) \geq \opt_\D - \epsilon \}\,.
\]
Moreover, let the set of optimal prices be
\[
T^* = T(0).
\]
Finally, define 
\[
\Delta(\epsilon) = 
\sup_{t \in T(\epsilon)}
\inf_{t' \in T^*}
\max\{|t - t'|, g_\D(t,t')\}\,.
\]

\paragraph{Tool I}
The following lemma will be crucial for the proof.
\begin{lemma}
The following hold true:
\begin{enumerate}
    \item  The set $T^*$ is non-empty.
    \item It holds that $\Delta(\epsilon) \in o(1).$
\end{enumerate}
\label{lemma:Sequences}
\end{lemma}

\begin{proof}
[Proof of \Cref{lemma:Sequences}]
First note that, for any sequence 
$t_1, t_2,\ldots$ with $t_i \D([t_i,1]) \to \sup_t t \D([t,1])$, 
compactness of $[0,1]$ implies there exists a convergent 
subsequence $t_{i_1}, t_{i_2}, \cdots$ 
with a limit point $t^*$ in $[0,1]$.
We will argue that $t^* \in T^*$.
Since any convergent sequence in $[0,1]$ contains 
a monotone subsequence, without loss of generality we may 
suppose $t_{i_1}, t_{i_2}, \ldots$ is either non-decreasing or non-increasing.
If $t_{i_1} \geq t_{i_2} \geq \cdots$, 
then we have 
\begin{equation*} 
t^* \D([t^*,1]) \geq t^* \D([t_{i_j},1]) \geq t_{i_j} \D([t_{i_j},1]) - (t_{i_j} - t^*) \xrightarrow[j \to \infty]{} \sup_{t} t \D([t,1]),
\end{equation*} 
so that $t^* \in T^*$.
On the other hand, 
if $t_{i_1} \leq t_{i_2} \leq \cdots$, 
then 
\begin{equation*}
t^* \D([t^*,1])
\geq t_{i_j} \D([t^*,1])
\geq t_{i_j} \D([t_{i_j},1]) - \D([t_{i_j},t^*))
\xrightarrow[j \to \infty]{} \sup_t t \D([t,1]),
\end{equation*}
so that again $t^* \in T^*$.
In particular, the above arguments imply $T^* \neq \emptyset$, 
since the definition of $\sup_t t \D([t,1])$ ensures there 
exists some sequence $t_1,t_2,\ldots$ as above.

Now for the remaining claim, for the purpose of obtaining a contradiction, suppose $\Delta(\epsilon) \neq o(1)$.
Since $\Delta(\epsilon)$ is monotone, this implies 
$\exists c > 0$ such that $\Delta(\epsilon) \geq c$ for all $\epsilon > 0$.
Let $t_1, t_2, \ldots$ be any sequence with $t_i \in T(1/i)$ 
such that every $i$ has
\begin{equation*}
\inf_{t \in T^*} \max\!\left\{ |t_i-t|, \min\{t_i,t\} \D([\min\{t_i,t\},\max\{t_i,t\})) \right\} > c/2.
\end{equation*}
By the above, there exists a convergent subsequence 
$t_{i_1}, t_{i_2}, \ldots$ 
with limit point $t^* \in T^*$: 
i.e., $|t_{i_j} - t^*| \xrightarrow[j \to \infty]{} 0$.
Again, since every convergent infinite sequence in $[0,1]$ contains 
an infinite monotone subsequence, 
without loss of generality we may suppose this $t_{i_j}$
subsequence is either non-increasing or non-decreasing.
If it is non-decreasing, we have 
$\D([\min\{t_{i_j},t^*\},\max\{t_{i_j},t^*\})) = \D([t_{i_j},t^*)) \xrightarrow[j \to \infty]{} 0$,
altogether contradicting the $c/2$ lower bound.
Otherwise, if it is non-increasing, we have
$\D([\min\{t_{i_j},t^*\},\max\{t_{i_j},t^*\})) = \D([t^*,t_{i_j}))$.
If $t_{i_j} = t^*$ eventually, the above equals $0$ for all large $j$, 
again contradicting the $c/2$ lower bound.
Otherwise, if every $t_{i_j} > t^*$, 
$t^* \D([t^*,t_{i_j})) \to t^* \D(\{t^*\})$.
We can further argue that in this case, 
$t^* \D(\{t^*\}) = 0$, 
since (recalling $t_{i_j} \in T(1/i_j)$)
\begin{equation*} 
t^* \D([t^*,1]) 
\leq \frac{1}{i_j} + t_{i_j} \D([t_{i_j},1])
\xrightarrow[j \to \infty]{} t^* \D((t^*,1]),
\end{equation*}
which implies $t^* \D(\{t^*\}) = 0$, 
and again we have contradicted the $c/2$ lower bound.
\end{proof}

\paragraph{Tool II}
We now proceed to the second key ingredient for the proof, which is a fine-grained localized concentration inequality. Again we will need some notation.

For each $t \in [0,1]$, 
let $t^*(t)$ denote any element of $T^*$ 
with 
\begin{align*} 
& \max\!\left\{ |t-t^*(t)|, g_\D(t,t^*(t)) \right\}
 \leq 2 \inf_{t' \in T^*} \max\!\left\{ |t-t'|, g_\D(t,t') \right\}.
\end{align*}
Such a $t^*(t)$ exists by definition of the infimum. If the right hand side is $0$, the same 
continuity-type arguments show that $t \in T^*$.

The following concentration inequality will be our second tool for the proof of \Cref{thm:FastRatesBounded}. Let us denote $f_t(x) = t \ind \{x \geq t\}.$ Note that $\E_x f_t(x) = t \Pr[x \geq t]$.
Moreover, let $\wh \E_S f_t = \frac{t}{|S|} \sum_{i \in S} \ind \{i \geq t\}$.

\begin{lemma}
[Bernstein Concentration, \cite{van2023weak}]
\label{lem:uniform-bernstein}
For any $\epsilon > 0$, \footnote{Let $\log(x) = \ln(\max\{x,e\})$.}
\begin{equation*}
\E_S\!\left[ \sup_{t \in T(\epsilon)} \left( \E[f_{t^*(t)}] - \E[f_t] \right) - \left( \wh{\E}_S[f_{t^*(t)}] - \wh{\E}_S[f_t] \right) \right] 
\leq C \sqrt{ \Delta(\epsilon) \frac{1}{n} \log\!\left(\frac{1}{\Delta(\epsilon)}\right) } + \frac{C\log(n)}{n},
\end{equation*}
where $C$ is a universal constant.
\end{lemma}

\paragraph{The Proof of \Cref{thm:FastRatesBounded}}

Let $\epsilon_n = C'\sqrt{\log(n)/n}$ 
for some appropriate absolute constant $C'$.
By uniform convergence of the CDF (i.e., \Cref{lem:dkw-ineq}), 
on an event $E_n$ of probability at least $1-1/n$, 
every function $f_t(x)=t \ind[x \geq t]$ 
has
\begin{equation*} 
\left|\wh{\E}[f_t]-\E[f_t]\right| < \epsilon_n/2\,,
\end{equation*}
where $\wh E$ is using the empirical CDF.
In particular, on this event, we have that the Empirical Revenue Maximizer satisfies
\begin{equation*} 
\E\!\left[f_{\hat{t}_n} \middle|S\right] > \max_{t'} \E[f_{t'}] - \epsilon_n,
\end{equation*}
so that $\hat{t}_n \in T(\epsilon_n)$. 
We can write the revenue gap as follows, recalling that 
$\hat{\E}_S[f_{t^*(\hat{t}_n)}] - \hat{\E}_S[f_{\hat{t}_n}] \leq 0$, 

\begin{align*}
&  \max_t t \D([t,1]) - \E\!\left[ \hat{t}_n \D([\hat{t}_n,1]) \right] 
 =\\
 & = \E_S 
 \!\left[ \E_{x \sim \D}\left[ f_{t^*(\wh{t}_n)}(x) \middle| S \right] - \E_{x \sim \D}\!\left[ f_{\hat{t}_n}(x) \middle| S \right] \right]
\\ & \leq \E\!\left[ \left( \E\left[ f_{t^*(\hat{t}_n)} \middle| S \right] - \E\!\left[ f_{\hat{t}_n} \middle| S \right] \right) \ind_{E_n} \right] + (1-\P(E_n))
\\ & \leq \E\!\left[ \left( \sup_{t \in T(\epsilon_n)} \left( \E\left[ f_{t^*(t)} \right] - \E\!\left[ f_{t} \right] \right) -  \left( \hat{\E}_S[ f_{t^*(t)} ] - \hat{\E}_S[f_t] \right) \right) \ind_{E_n} \right] + (1-\Pr(E_n))
\\ & \leq \E\!\left[ \sup_{t \in T(\epsilon_n)} \left( \E\left[ f_{t^*(t)} \right] - \E\!\left[ f_{t} \right] \right) -  \left( \hat{\E}_S[ f_{t^*(t)} ] - \hat{\E}_S[f_t] \right) \right] + (1-\Pr(E_n))
\\ & \leq C \sqrt{\Delta(\epsilon_n) \frac{1}{n}\log\!\left(\frac{1}{\Delta(\epsilon_n)}\right)} + \frac{1+C\log(n)}{n}.
\end{align*}

In the above, the first equality follows by definition of $f_t(x)$, the first inequality follows by conditioning on the event $E_n$, the second inequality follows because (i) we take the supremum over $t \in T(\epsilon_n),$ a set that contains $\wh t_n$ under $E_n$ and (ii) $\hat{\E}_S[f_{t^*(\hat{t}_n)}] - \hat{\E}_S[f_{\hat{t}_n}] \leq 0$. The third inequality follows since $\Pr[E_n] \leq 1$ and the last inequality follows because $\Pr[E_n] = 1-1/n$ and from \Cref{lem:uniform-bernstein}. 

Noting that 
$\lim_{\epsilon \to 0} \sqrt{ \Delta(\epsilon) \log\!\left(\frac{1}{\Delta(\epsilon)}\right)} = 0$, 
the above is $o(n^{-1/2})$, which completes the proof.

\section{Bounds on Discrete and Closed Support}\label{sec:bounds-discrete-closed}
We first give the definition of closed and discrete sets, which we will use
in our result.
\begin{defn}[Closed and Discrete Set]\label{def:closed-and-discrete}
    A set $S \subseteq \R_+$ is closed and discrete if 
    \textbf{i)} it contains all of its accumulation points (closed),
    and \textbf{ii)} for every point \(x \in S\) there is some \(\varepsilon_x>0\) such that
\[
(x-\varepsilon_x, x+\varepsilon_x)\cap S = \{x\}.
\]
\end{defn}

\subsection{Near Exponential Rates}
\label{proof:Exponential}

In this setting we prove the result regarding our Structural Revenue Maximization algorithm.

\smallskip

\begin{mybox}
    \begin{center}
        Structural Empirical Revenue Maximization
    \end{center}
    \begin{enumerate}
        \item \textbf{Input}: Draw $S \sim \D^n$, non-decreasing function $g : \N \to \N, g(n) \in o(n)$
        \item \textbf{Define function $f(n) \in o(1)$ such that $f^2(n)\cdot n \geq g(n)$, for $n$ sufficiently large}
        \item Using $S,$ create the empirical distribution $\hat \D_n$
        \item Sort the elements of $S$ in non-decreasing order $p_{1}, p_2,\ldots,p_n$
        \item Let \(i^*\) be the largest index in \([n]\) such that
        \begin{equation}
        \rev_{\hat \D_n}(p_{i^*})
        >
        \rev_{\hat \D_n}(p_j)
        +
        (p_j+p_{i^*})f(n)
        \qquad
        \text{for all } j<i^*.
        \label{eq:winner}
        \end{equation}
        
        \item \textbf{return} the price $t_n := p_{i^*}$
    \end{enumerate}
\end{mybox}

In words, the algorithm uses the training set to create the empirical distribution and sorts the observed points in increasing order. It then finds the rightmost point of the sorted list that beats all prior points in empirical revenue with confidence, i.e., it satisfies \eqref{eq:winner} against all points that precede it in the enumeration. Importantly the gap between the price $p_{i^*}$ and $p_j$ in \eqref{eq:winner} has to scale with $(p_j + p_{i^*}) \cdot f(n)$, i.e., it goes to 0 as $n\rightarrow\infty$ since $f(n)$ is some $o(1)$ function picked by the algorithm. Different choices of $f(n)$ give different rates $e^{-o(n)}$. The reason for selecting this function is that the algorithm does not know what is the separation between the optimal price and the other prices; the only thing it knows is that the gap is positive. Hence, it should use $f(n)$ to make the gap more  and more fine-grained until \eqref{eq:winner} is satisfied. 

\begin{rem}
We note that the $e^{-o(n)}$ rate that we get in this setting is conceptually different from the $o(n^{-1/2})$ rate that we get in the bounded support case. 

\begin{enumerate}
    \item In the bounded support case, we show that we can get a rate slightly faster than $1/\sqrt{n}$ but arbitrarily close to it.

    \item Here, the rate $e^{-o(n)}$ is slower than exponential and the $o(n)$ rate depends on how the algorithm chooses the function $f(n)$.
\end{enumerate}
\end{rem}

\begin{lemma}
For any $g(n) \in o(n)$, the class of valuation distributions $\mathbb D$ 
    supported on (potentially different) closed and discrete subsets of  $\R_+$
    and have optimal revenue that is attained by some price,
    is universally learnable at a rate $e^{-g(n)}.$
\end{lemma}

\begin{proof}
    Let $\D$ be the target distribution and $\mathcal{X}$ be its support.
    It is an immediate consequence of the discreteness property
    that $\mathcal{X}$ is countable. 
    Moreover, since $\mathcal{X}$ is also closed, it admits
    an enumeration in increasing order.
    Let $p_1 < p_2,\ldots,$ be an increasing enumeration of the support $\mathcal{X}$
    and also $p^* \in \R_+$ be an optimal price. Moreover, since $g(n) \in o(n)$
    there exists some function $f(n) \in o(1)$ and constants $c, C$ such that
    $2n \cdot f^2(n) \geq C \cdot g(c n), \forall n \in \N.$
    
    We first observe that
    there exists some $i^* \in \N$ such that $p_{i^*}$ is an optimal price
    and for all $j \in \N, j < i^*,$ it holds that $p_j$ is not an optimal price.
    To see that, notice that if the set is bounded and closed and discrete, then it is finite, hence it is immediate that it admits an optimal price 
    on its support. Otherwise, there exists some $p_i > p^*, p_i \in \mathcal{X}.$
    Then, $\mathcal{X}_i \coloneqq [0, p_i] \cap \mathcal{X}$ contains finitely
    many points, and at least one point that is (strictly) greater than $p^*.$
    If $p^* \notin \mathcal{X}_i,$ then there is (at least) one $p_j \in \mathcal{X}_i$ such that the revenue of $p_j$ is strictly greater
    than the revenue of $p^*$. This is because $p_j$ has the same probability
    of sale but generates strictly higher revenue under sale. Thus, 
    $p^*$ is in the support of $\D$ and there are finitely many points before it,
    hence we can define $i^*$ to be the smallest index of an optimal price.
    Let also $m^*$ be the probability mass of $p_{i^*}$ under $\D.$
    We notice that with probability at least $1-(1-m^*)^n \geq 1-e^{-m^* n}$
    if we draw $n$ i.i.d. samples from $\D$ the element $p_{i^*}$ will appear
    at least once in the dataset. We call this event $E_1.$
    We now consider two cases for the rest of the analysis.
    
    \paragraph{Case 1: $i^*=1$.} We first assume that $p_{i^*}$ is the first
    price in the increasing enumeration. Then, we only need to show that
    with high probability none of the subsequent prices will be selected
    by the algorithm. 
        Define the uniform DKW event
    \[
    E_{\mathrm{DKW}}
    :=
    \left\{
    \sup_{p\in\R_+}\abs{F_n(p)-F(p)}\le f(n)
    \right\}.
    \]
    By the DKW inequality,
    \[
    \Pr(E_{\mathrm{DKW}}^c)\le 2e^{-2nf(n)^2}\le 2e^{-C g(cn)}
    \]
    for suitable constants \(C,c>0\). On \(E_{\mathrm{DKW}}\), for every price \(p\),
    \[
    \abs{\rev_{\hat\D_n}(p)-\rev_\D(p)}\le p f(n).
    \]
    Therefore, on \(E_1\cap E_{\mathrm{DKW}}\), for every observed price \(p_j>p_{i^*}\),
    \begin{align*}
    \rev_{\hat\D_n}(p_j)
    &\le \rev_\D(p_j)+p_j f(n)\\
    &\le \rev_\D(p_{i^*})+p_j f(n)\\
    &\le \rev_{\hat\D_n}(p_{i^*})+(p_j+p_{i^*})f(n).
    \end{align*}
    Hence no observed price to the right of \(p_{i^*}\) can satisfy the winning condition \eqref{eq:winner}. Thus, conditioned on \(E_1\), no later price is selected except on an event of probability at most \(2e^{-C g(cn)}\). 
    
    Thus, putting everything together, 
    \[
        \E[\rev(t_n)] \geq \opt\cdot(1- 2e^{-C\cdot g(cn)} - e^{-m^* n}) = \opt - 2\opt e^{-C\cdot g(cn)} - \opt  e^{-m^* n}\,,
    \]
    which proves the claim.

    \paragraph{Case 2: $i^* > 1$} In this case, the previous 
        analysis, using the same uniform DKW event \(E_{\mathrm{DKW}}\), shows that with probability at least 
    \(1-2e^{-C\cdot g(cn)}\), no price \(j>i^*\) will be selected by the algorithm.
    Next, we argue that for sufficiently large $n,$ with high probability, 
    no price $p_{j'}, j' < i^*$ will be selected. This part of the proof
    resembles the proof of the finite support case. We need 
    to show that $\rev_{\hat \D_n}(p_{i^*}) > \rev_{\hat \D_n}(p_{j'}) + (p_{j'}+p_{i^*})f(n),$ for all $j' < i^*.$
    Since $p_{i^*}$ is the first optimal price, it holds that
    $\rev_{\D}(p_{i^*}) > \rev_{\D}(p_{j'}), \forall j' < i^*$ 
    and since there are finitely many such $j'$ there exists some number
    $d > 0$ such that $\rev_{\D}(p_{i^*}) > \rev_{\D}(p_{j'}) + d, \forall j' < i^*.$
    Again, using the DKW inequality we have that
       \begin{align*}
                \Pr[\abs{\rev_\D(p_{i^*}) - \rev_{\hat \D_n}(p_{i^*})} > p_{i^*}f(n)] &\leq 2e^{-C\cdot g(cn)}\\
                \Pr[\abs{\rev_\D(p_{j'}) - \rev_{\hat \D_n}(p_{j'})} > p_{j'}f(n)] &\leq 2e^{-C\cdot g(cn)} \,.
    \end{align*}
    Let us now condition on this event.
    Then, we have that
    \begin{align*}
        \rev_{\hat \D_n}(p_{i^*}) &\geq \rev_\D(p_{i^*}) -  p_{i^*}f(n) & \text{(Conditioned event)}\\
        &\geq \rev_\D(p_{j'}) + d -  p_{i^*}f(n) &\text{(Optimality of rev.)}\\
        &\geq \rev_{\hat \D_n}(p_{j'}) + d - p_{j'}f(n) -p_{i^*}f(n) & \text{(Conditioned event)}\\
        &\geq  \rev_{\hat \D_n}(p_{j'}) + \frac{d}{2} - p_{j'}f(n) -p_{i^*}f(n) + 2p_{j'}f(n) + 2p_{i^*}f(n) &(\text{$n$ sufficiently large}) \\
        &>  \rev_{\hat \D_n}(p_{j'}) +  p_{j'}f(n) + p_{i^*}f(n) \,,
    \end{align*}
    for all $n$ sufficiently large, since $f(n) \in o(1).$
        Thus, under this event the algorithm chooses \(p_{i^*}\). Hence, 
    for all sufficiently large \(n\), the revenue satisfies
    \[
        \E[\rev(t_n)]
        \geq
        \opt_\D\cdot\left(1- 2e^{-C\cdot g(cn)} - e^{-m^* n}\right)
        =
        \opt_\D
        -
        2\opt_\D e^{-C\cdot g(cn)}
        -
        \opt_\D e^{-m^* n}.
    \]
    Since \(g(n)=o(n)\), the term \(e^{-m^*n}\) can be absorbed into the \(e^{-C g(cn)}\) term after adjusting the distribution-dependent constants. This proves the claimed \(e^{-g(cn)}\)-type rate.
\end{proof}

{
\subsection{Lower Bound of ERM for Discrete and Closed Distributions}\label{sec:erm-proof}
In this section we give a formal proof of the fact that ERM does not converge to the optimal revenue.

\begin{proof}[Proof of \cref{Thm:Discrete} (item ii)]
    Let the support of the distribution $\D$ be the set $X = \{1\} \cup \{4^k \mid k \in \N, k \ge 1\}$. This set is discrete and closed. The probability mass function (PMF) is defined based on its tail probabilities.

Let the revenue for the optimal price $p^*=1$ be $\rev_{\D}(1) = 1$. This implies $\Pr(V \ge 1) = 1$.
For the challenger prices $p_k = 4^k$, we set their revenue to be constant:
\[ \rev_{\D}(4^k) = \frac{1}{2} \quad \text{for all } k \ge 1 \]
This implies that the tail probability must be $\Pr(V \ge 4^k) = \frac{\rev_{\D}(4^k)}{4^k} = \frac{1}{2 \cdot 4^k}$.

From these tail probabilities, we can define the PMF for the challenger prices:
\begin{align*}
\Pr(V=4^k) &= \Pr(V \ge 4^k) - \Pr(V \ge 4^{k+1}) \\
&= \frac{1}{2 \cdot 4^k} - \frac{1}{2 \cdot 4^{k+1}} = \frac{4-1}{2 \cdot 4^{k+1}} = \frac{3}{2 \cdot 4^{k+1}}
\end{align*}
The total mass assigned to all challengers is:
\[ \sum_{k=1}^{\infty} \Pr(V=4^k) = \sum_{k=1}^{\infty} \frac{3}{2 \cdot 4^{k+1}} = \frac{3}{2} \sum_{j=2}^{\infty} \left(\frac{1}{4}\right)^j = \frac{3}{2} \left( \frac{(1/4)^2}{1-1/4} \right) = \frac{3}{2} \left( \frac{1/16}{3/4} \right) = \frac{1}{8} \]
The remaining mass is assigned to the optimal price $p^*=1$:
\[ \Pr(V=1) = 1 - \sum_{k=1}^{\infty} \Pr(V=4^k) = 1 - \frac{1}{8} = \frac{7}{8} \]
This construction defines a valid PMF. The optimal revenue is unique: $\opt_{\D}=1$ at $p^*=1$. For any challenger price $p_k=4^k$, the regret is constant: $\opt_{\D} - \rev_{\D}(4^k) = 1 - 1/2 = 1/2$.

We analyze the ERM algorithm for the subsequence of sample sizes $n_k = 4^k$. The proof rests on showing that for this subsequence, an error occurs with a constant, non-zero probability.

By construction, all suboptimal prices in the support, $\{4^k\}_{k \ge 1}$, yield the same true revenue of $1/2$. Therefore, if the ERM algorithm selects any price other than the optimum $p^*=1$, the regret incurred is exactly $\opt_{\D} - 1/2 = 1/2$.

Let us define a sufficient condition for such an error to occur. Let $\mathcal{E}_k$ be the event that, for a sample of size $n_k=4^k$, at least two samples have a value of $4^k$ or greater.
If this event $\mathcal{E}_k$ occurs, we can lower-bound the empirical revenue of the price $p_k=4^k$:
\[ \rev_{\hat{\D}}(4^k) \ge 4^k \cdot \frac{2}{n_k} = 4^k \cdot \frac{2}{4^k} = 2 \]
The empirical revenue of the optimal price $p^*=1$ is always $\rev_{\hat{\D}}(1) = 1$ (since all samples are at least 1).
Since $\rev_{\hat{\D}}(4^k) \ge 2 > 1 = \rev_{\hat{\D}}(1)$, the event $\mathcal{E}_k$ guarantees that ERM will not select the optimal price $p^*=1$. Thus, on this event, an error is made and the regret is exactly $1/2$.

The expected regret can therefore be lower-bounded by:
\[ \E[\text{Regret}(n_k)] \ge (\text{Regret on event } \mathcal{E}_k) \cdot \Pr(\mathcal{E}_k) = \frac{1}{2} \cdot \Pr(\mathcal{E}_k) \]
It remains to show that $\Pr(\mathcal{E}_k)$ is lower-bounded by a constant for large $k$. Let $N_k$ be the number of samples with value at least $4^k$. $N_k$ follows a Binomial distribution, $N_k \sim \text{Binomial}(n_k, p'_k)$, where $p'_k = \Pr(V \ge 4^k) = \frac{1}{2 \cdot 4^k}$. The event $\mathcal{E}_k$ is equivalent to $\{N_k \ge 2\}$.

The expected number of such samples is $\lambda_k = n_k \cdot p'_k = 4^k \cdot \frac{1}{2 \cdot 4^k} = 1/2$.
As $k \to \infty$, the Binomial distribution $\text{Binomial}(n_k, p'_k)$ converges to a Poisson distribution with mean $\lambda = \lim_{k\to\infty} \lambda_k = 1/2$.
Therefore,
\begin{align*}
\lim_{k\to\infty} \Pr(\mathcal{E}_k) &= \lim_{k\to\infty} \Pr(N_k \ge 2) = \Pr(\text{Poisson}(1/2) \ge 2) \\
&= 1 - \Pr(\text{Poisson}(1/2)=0) - \Pr(\text{Poisson}(1/2)=1) \\
&= 1 - e^{-1/2} - \frac{1}{2}e^{-1/2} = 1 - \frac{3}{2}e^{-1/2} \approx 0.0902
\end{align*}
This limit is a positive constant. Let $c_1 = 1 - \frac{3}{2}e^{-1/2}$. For sufficiently large $k$, we have $\Pr(\mathcal{E}_k) \ge c_1/2 > 0$.

The expected regret for sample size $n_k$ is thus lower-bounded by:
\[ \E[\text{Regret}(n_k)] \ge \frac{1}{2} \cdot \Pr(\mathcal{E}_k) \ge \frac{c_1}{4} \]
Since $c_1/4$ is a positive constant, we have shown that for the infinite subsequence of sample sizes $\{n_k=4^k\}$, the expected regret is lower-bounded by a constant. This proves that ERM is not a consistent learner for this distribution.

\end{proof}
}

\section{Finite Support}
\label{sketch:ExponentialFinite}

In this section we describe our results for exponential rates upper and lower bounds under finite support.

\subsection{Exponential Rates Upper Bound}\label{sec:exp-rates-upper-bound-finite}

  We will show that there exists a learning
algorithm $t_n$ that achieves exponential rates for the class of valuation distributions supported on a finite set
of prices. 

\noindent \textbf{The Proof of \Cref{Thm:Finite}} ~Our algorithm for the finite case is simply Empirical Revenue Maximization. Given $n$ i.i.d. samples, we compute the empirical CDF and output the price $\wh t_n$ that achieves maximum empirical revenue.

\begin{mybox}
    \begin{center}
        The Algorithm of \Cref{Thm:Finite} (Empirical Revenue Maximization)
    \end{center}
    \begin{enumerate}
        \item \textbf{Input}: Draw $S \sim \D^n$

        \item Construct the empirical distribution $\wh \D_n$
        
        \item \textbf{return} the price $\wh t_n = \argmax_{t \in S} \rev_{\wh \D_n}(t)$.
    \end{enumerate}
\end{mybox}

    Let \(\M = \{p_1,\ldots,p_m\}\) be the prices in the support. Without loss of generality, assume
    \(p_1 < p_2 < \cdots < p_m\). Let \(\D \in \mathbb D\) be the underlying valuation distribution, let \(F\) be its CDF, and let
    \[
    \M^* = \argmax_{i\in[m]} \rev_\D(p_i)
    \]
    be the set of indices that achieve the optimal revenue. If \(\M^*=[m]\), then every support price is optimal, and since the algorithm outputs an element of the sample, its revenue gap is zero. Hence, assume \(\M^*\neq[m]\), and define
    \[
    d^*
    =
    \max_{i\in[m]}\rev_\D(p_i)
    -
    \max_{j\in[m]\setminus \M^*}\rev_\D(p_j)
    >0.
    \]
    Fix an optimal index \(i^\star\in\M^*\), and let
    \[
    m^\star := \Pr_{v\sim\D}[v=p_{i^\star}]>0.
    \]
    Let \(E_\star\) be the event that \(p_{i^\star}\) appears at least once in the sample \(S\). Then
    \[
    \Pr(E_\star^c)=(1-m^\star)^n\le e^{-m^\star n}.
    \]

    Define the uniform concentration event
    \[
    E_{\mathrm{DKW}}
    :=
    \left\{
    \sup_{x\in\R}\abs{F_n(x)-F(x)}
    \le
    \frac{d^*}{4p_m}
    \right\}.
    \]
    By the DKW inequality,
    \[
    \Pr(E_{\mathrm{DKW}}^c)
    \le
    2e^{-2n(d^*/(4p_m))^2}
    =
    2e^{-n(d^*)^2/(8p_m^2)}.
    \]
    On the event \(E_{\mathrm{DKW}}\), every support price \(p_j\) satisfies
    \[
    \abs{\rev_{\hat\D_n}(p_j)-\rev_\D(p_j)}
    \le
    p_j\frac{d^*}{4p_m}
    \le
    \frac{d^*}{4}.
    \]
    Therefore, on \(E_\star\cap E_{\mathrm{DKW}}\), the optimal price \(p_{i^\star}\) is available to the algorithm, and for every non-optimal support price \(p_j\),
    \[
    \rev_{\hat\D_n}(p_{i^\star})
    \ge
    \opt_\D-\frac{d^*}{4}
    >
    \opt_\D-d^*+\frac{d^*}{4}
    \ge
    \rev_{\hat\D_n}(p_j).
    \]
    Hence the ERM output is optimal on \(E_\star\cap E_{\mathrm{DKW}}\). Thus,
    \[
    \Pr[\wh t_n\notin \M^*]
    \le
    e^{-m^\star n}
    +
    2e^{-n(d^*)^2/(8p_m^2)}.
    \]
    Since the revenue gap is always at most \(\opt_\D\), we get
    \[
    \epsilon_n(\wh t_n,\D)
    \le
    \opt_\D
    \left(
    e^{-m^\star n}
    +
    2e^{-n(d^*)^2/(8p_m^2)}
    \right)
    \le
    c_1 e^{-c_2 n}
    \]
    for some distribution-dependent constants \(c_1,c_2>0\).

\begin{rem}
[Interpretation of the Constants]
\label{rem:exp-rates-finite-prices}
    Notice that the distribution dependent constant for the finite case have a natural interpretation: they depend
    on the largest price in $\M$, the maximum
    revenue achievable by $\M$, {the mass of the optimal price},
    and
    the gap between the best
    revenue and the second-best revenue achievable by $\M$. 
\end{rem}

\subsection{Exponential Rates Lower Bound}\label{proof:expr-rates-lower-bound}
In this section, we show the following.

\begin{lemma}
The class of valuation distributions $\mathbb D $ 
    supported on a {finite} domain $\mathcal{X} \subset \R_+, \abs{\mathcal X} \geq 2,$
     is not universally learnable at a rate faster than exponential.
\end{lemma}

\begin{proof}
    Since the support has size at least two, there are two prices $p, p' \in \mathcal X$ with $p < p'.$ Let us define two distribution $\D_p, \D_{p'}$ with 

\[\left\{\Pr_{v \sim \D_{p}}[v = p] = 1\right\}, \left\{\Pr_{v \sim \D_{p'}}[v = p] = q, \Pr_{v \sim \D_{p'}}[v = p'] = 1-q\right\}\,,\]
for some $q \in (0,1)$ to be decided.

The optimal price of $\D_p$ is $p$ with revenue $\opt_{\D_p} = p$.

We pick $q$ such that the optimal price of $D_{p'}$ is $p' > p$. Since
\[
\rev_{\D_{p'}}(p) = p \Pr_{v}[v \geq p] = p\,,
\rev_{\D_{p'}}(p') = p' \Pr_{v}[v \geq p'] = p'(1-q)\,.
\]
Hence, let us set $q < 1$ such that $p'(1-q) = c \cdot p,$ for $c > 1.$

Fix a learning algorithm $\{t_n\}_{n \in \N}$. Without loss assume that $t_n$ is a randomized mapping that maps training sets of size $n$ to $\Delta(\{p,p'\})$ (since any other output is sub-optimal, so showing a lower bound for such cases is easier). 

For any $n$, assume that 
$ a_n = \Pr\left[t_n(p,\ldots,p) = p \right]$ where the probability is with respect to the randomness of $t_n.$
In words, $a_n$ is the likelihood that the algorithm $t_n$, given a dataset that contains \(n\) copies of \(p\), predicts the price \(p\).

Consider two cases. If \(a_n \leq \nicefrac{1}{2}\), then under the input \(S_n=(p,\ldots,p)\), we have
\[
    \E_{t_n(p,\ldots,p)}
    \left[
    \rev_{\D_p}(t_n(p,\ldots,p))
    \right]
    =
    a_n p + (1-a_n)\cdot 0
    =
    a_n p
    \leq
    \frac{p}{2}.
\]
Note that
\[
    \opt_{\D_p} = p \,,
\]
thus
\[
   \opt_{\D_p}- \E_{t_n(p,...,p)}\rev_{\D_p}(t_n(p,\ldots,p)) \geq \frac{p}{2} \,.
\]
On the other hand, if $a_n \geq 1/2$ then
\[
    \E_{t_n(p,...,p)}\rev_{\D_{p'}}(t_n(p,\ldots,p)) \leq \frac{p}{2} + \frac{c\cdot p}{2} \leq \frac{c+1}{2} \cdot p \,,
\]
since at least half of the mass of $t_n(p,...,p)$ is on $p$. Recall
\[
    \opt_{\D_{p'}} = cp, c > 1  \,,
\]
by the choice of $q.$
Thus, we have that
\[
    \opt_{\D_{p'}} - \rev_{\D_{p'}}(t_n(p,\ldots,p)) \geq \frac{c-1}{2}\cdot p > 0 \,.
\]
Thus, conditioned on the event that the sample is \((p,\ldots,p)\), the learning algorithm has a constant loss with respect to at least one of the two distributions for every \(n\in\N\).

Let
\[
A:=\{n\in\N: a_n\le 1/2\}.
\]
If \(A\) is infinite, then for every \(n\in A\), under \(\D_p\) the sample is always \((p,\ldots,p)\), and therefore
\[
\epsilon_n(t_n,\D_p)
\ge
\frac{p}{2}.
\]
This is stronger than an exponential lower bound.

Otherwise, \(A^c\) is infinite. For every \(n\in A^c\), the previous argument gives a constant conditional loss under \(\D_{p'}\) on the event \(S=(p,\ldots,p)\). Under \(\D_{p'}\), this event has probability \(q^n\). Therefore, for every \(n\in A^c\),
\[
\epsilon_n(t_n,\D_{p'})
\ge
q^n\cdot \frac{c-1}{2}p
=
\frac{c-1}{2}p\cdot e^{-n\log(1/q)}.
\]
In either case, there is a fixed distribution in \(\{\D_p,\D_{p'}\}\) for which the learner's error is bounded below by an exponential rate along infinitely many sample sizes. Hence no algorithm can achieve a rate faster than exponential on this class.
\end{proof}

\newpage
\bibliographystyle{ACM-Reference-Format}
\bibliography{ec_submission/sample-bibliography}

\newpage 

\appendix

\section{Omitted Details from \Cref{sec:intro}}\label{sec:ex}
\subsection{Additive vs. Multiplicative Error}\label{sec:additive-mult-error}
{In this section we show that achieving a multiplicative error
is equivalent to achieving an additive error in the universal
rates setting.

\begin{claim}\label{clm:multplicative-additve}
    Let $\mathbb D$ be a class of distributions and $R$ be some rate. Then, an algorithm $t_n$ achieves universal rate $R$ 
    with respect to the multiplicative error definition
    if and only if it achieves  universal rate $R$ 
    with respect to the additive error definition.
\end{claim}

\begin{proof}
    Let $t_n$ be an algorithm that achieves  universal rate $R$ 
    with respect to the multiplicative error definition, i.e., for all $\D \in \mathbb{D}$ it holds that
    \[
        \E[\rev(t_n)] \geq (1 - CR(cn)) \opt_\D, \forall n \in \N \,.
    \]
    Then,
    \[
        \E[\rev (t_n)] \geq \opt_\D - C\opt_\D \cdot R(cn) =  \opt_\D - C' R(cn), \forall n \in \N \,,
    \]
    hence the algorithm achieves universal rate $R$ with respect to the additive error.

    Conversely, assume that $t_n$ is an algorithm that achieves  universal rate $R$ 
    with respect to the additive error definition, i.e., for all $\D \in \mathbb{D}$ it holds that
    \[
        \E[\rev(t_n)] \geq \opt_\D - CR(cn), \forall n \in \N \,.
    \]
    Then,
        \[
        \E[\rev(t_n)] \geq \opt_\D - CR(cn) = (1-\nicefrac{C}{\opt_\D} \cdot R(cn)) \opt_\D = (1 - C'R(cn)) \opt_\D, \forall n \in \N \,,
    \]
    hence the algorithm achieves universal rate $R$ with respect to the multiplicative error.

\end{proof}
}

\subsection{Distributions Without Optimal Price}
\label{section:Examples}
\label{sec:ex-ditr-no-opt-price}
We now provide examples of some distributions that do not admit
an optimal price, even though they have finite revenue.


\begin{example}[Regular Distribution without Optimal Price]
    Consider a distribution $\D$ whose CDF is $F_\D(v) = 1 - \frac{1}{v+1}$
    and PDF is $f_\D(v) = \frac{1}{(v+1)^2}.$ Then, 
    $v - \frac{1-F(v)}{f(v)} = -1,$ hence non-decreasing. 
    Moreover $p\cdot (1-F_\D(p)) = \nicefrac{p}{p+1},$ thus $\opt_\D = 1$
    but it is not achieved by any price.
\end{example}

The following slightly more complicated example provides such a distribution
whose virtual value is strictly increasing.

\begin{example}
    Consider a distribution $\D$ whose CDF is $F_\D(v) = 1 - \frac{1}{2}\left( \frac{1}{v+1} + \frac{1}{(v+1)^2}\right), v \geq 0,$
    and PDF is $f_\D(v) = \frac{v+3}{2(v+1)^3}.$ Then, $p(1-F_\D(p)) = \frac{p}{2}\left( \frac{1}{p+1} + \frac{1}{(p+1)^2}\right) \rightarrow \frac{1}{2},$
    hence $\opt_\D = \frac{1}{2},$ but this is not achieved by any price. 
    Moreover, $v - \frac{1-F(v)}{f(v)} = -\frac{2}{v+3},$ hence
    the virtual value is increasing.
 \end{example}

 Finally, we also provide a simple example of a discrete and closed
 set for which this holds.

 \begin{example}
    Consider a distribution $\D$ supported on $\N$ where $Pr[X \geq n]=\frac{2}{n+1}.$
    It is not hard to see that this is well-defined. Moreover, similarly as above, $\opt_\D = 2,$ but it is not achieved by any price.
 \end{example}

 \subsection{Concentration Inequality}
Below, we provide a concentration inequality we use throughout our work.

\begin{lem}[DKW Inequality]\label{lem:dkw-ineq}
    Let $F$ be the CDF of some distribution $\D$ over $\R$ and
    $F_n$ the empirical CDF of $\D$ obtained by 
    taking $n$ i.i.d. samples from $\D.$
    Then,
    \[
        \Pr[\sup_{x \in \R}\abs{F_n(x) - F(x)} > \varepsilon] \leq 2e^{-2n\varepsilon^2}, \forall \varepsilon > 0.
    \]
\end{lem}

\end{document}